\documentclass[11pt]{article} 
\usepackage{hyperref}
\usepackage{url}
\usepackage{smile}
\usepackage{graphicx} 
\usepackage{algorithm}
\usepackage{algorithmic}
\usepackage{todonotes}
\usepackage{epstopdf}
\usepackage[margin=1in]{geometry}
\usepackage[normalem]{ulem}
\usepackage[export]{adjustbox}
\usepackage{mathtools, cuted}
\usepackage{natbib}
\usepackage{bbm}
\usepackage{wrapfig}
\usepackage{subcaption}
\usepackage{caption}
\usepackage{enumitem}
\usepackage{mysymbols}
\usepackage{mathtools}\usepackage{cleveref}

\newcommand{\oa}{\overline{a}}
\newcommand{\indic}{\mathbbm{1}}

\newcommand{\pLG}{\overline{g}^0}
\newcommand{\pG}{\overline{g} }

\newcommand{\init}{\mathrm{init}}
\newcommand{\iid}{\overset{\mathrm{i.i.d.}}{\sim} }
\newcommand{\nom}{\nonumber}

\hypersetup{
    colorlinks=true,
    linkcolor=blue,
    anchorcolor=blue, 
    citecolor=blue,
}

\usepackage{kpfonts}
\DeclareMathAlphabet{\mathsf}{OT1}{cmss}{m}{n}
\SetMathAlphabet{\mathsf}{bold}{OT1}{cmss}{bx}{n}

\providecommand{\norm}[1]{\|#1\|}

\newcommand{\as}{\mathbf{s}}

\DeclarePairedDelimiter\abs{\lvert}{\rvert}

\DeclarePairedDelimiterX{\infdivx}[2]{(}{)}{%
  #1\;\delimsize\|\;#2%
}

\renewcommand{\algorithmiccomment}[1]{\textit{#1}}

\usepackage{mathtools}

\newcommand\blfootnote[1]{%
  \begingroup
  \renewcommand\thefootnote{}\footnote{#1}%
  \addtocounter{footnote}{-1}%
  \endgroup
}

\begin{document}

\title{\huge \bf{Permutation Invariant Policy Optimization for Mean-Field Multi-Agent Reinforcement Learning: A Principled Approach}}

\author{
Yan Li$^\dagger$,  Lingxiao Wang$^\diamond$, Jiachen Yang$^\dagger$, Ethan Wang$^\dagger$, \\
Zhaoran Wang$^\diamond$, Tuo Zhao$^\dagger$, Hongyuan Zha$^*$\blfootnote{
Yan Li, Jiachen Yang, Ethan Wang, and Tuo Zhao are affiliated with Georgia Institute of Technology; Lingxiao Wang, Zhaoran Wang are affiliated with Northwestern University; Hongyuan Zha is affiliated with The Chinese University of Hong Kong, Shenzhen.  Tuo Zhao is the corresponding author; Email: tourzhao@gatech.edu.}
}
\date{\vspace{-5ex}}

\maketitle

\begin{abstract}
Multi-agent reinforcement learning (MARL) becomes more challenging in the presence of more agents, as the capacity of the joint state and action spaces grows exponentially in the number of agents. To address such a challenge of scale, we identify a class of cooperative MARL problems with permutation invariance, and formulate it as  mean-field Markov decision processes (MDP). To exploit the permutation invariance therein, we propose the mean-field proximal policy optimization (MF-PPO) algorithm, at the core of which is a permutation-invariant actor-critic neural architecture. 
We prove that MF-PPO attains the globally optimal policy at a sublinear rate of convergence. Moreover, its sample complexity is independent of the number of agents. We validate the theoretical advantages of MF-PPO with numerical experiments in the multi-agent particle environment (MPE). In particular, we show that the inductive bias introduced by the permutation-invariant neural architecture enables MF-PPO to outperform existing competitors with a smaller number of model parameters, which is the key to its generalization performance. 
\end{abstract}

\vspace{-0.15in}
\section{Introduction}
\vspace{-0.05in}

Multi-Agent Reinforcement Learning \citep{littman1994markov,zhang2019multi} generalizes Reinforcement Learning \citep{sutton2018reinforcement} to address the sequential decision-making problem of multiple agents maximizing their own long term rewards while interacting with one another in a common environment.
With breakthroughs in deep learning, MARL algorithms equipped with deep neural networks have seen significant empirical successes in various domains,
including simulated autonomous driving \citep{shalev2016safe}, multi-agent robotic control \citep{mataric1997reinforcement,kober2013reinforcement}, and E-sports \citep{vinyals2019grandmaster}.

Despite tremendous successes, MARL is notoriously hard to scale to the many-agent setting, as the size of the state-action space grows exponentially with respect to the number of agents.
This phenomenon is recently described as the curse of many agents \citep{menda2018deep}.
To tackle this challenge,
we focus on \textit{cooperative MARL},
where agents work together to maximize their team reward \citep{panait2005cooperative}.  
We identify and exploit a key property of cooperative MARL with homogeneous agents, namely the \textit{invariance with respect to the permutation of agents}.
Such permutation invariance can be found in many real-world scenarios with homogeneous agents, such as distributed control of multiple autonomous vehicles and team sports \citep{cao2013overview,kalyanakrishnan2006half}, 
and also in the case of heterogeneous groups where invariance holds within each group \citep{liu2019pic}.
More importantly, we further find that permutation invariance has significant practical implications, as the optimal value functions remains invariant when permuting the joint state-action pairs. 
Such an observation strongly advocates a permutation invariant design for learning, which helps reduce the effective search space of the policy/value functions from exponential dependence on the number of agents to polynomial dependence.

Several empirical methods have been proposed to incorporate permutation invariance into solving MARL problems.
  \citet{liu2019pic} implement a permutation invariant critic based on Graph Convolutional Network (GCN) \citep{kipf2016semi}.
 \citet{sunehag2017value} propose value decomposition, which together with parameter sharing, leads to a joint critic network that is permutation invariant over agents.
 While these methods are based on heuristics, we are  the first to provide theoretical principles for introducing permutation invariance as an inductive bias 
 for learning value functions and policy in homogeneous systems. 
In addition, we adopt the DeepSet \citep{zaheer2017deep} architecture, which is well suited for handling homogeneity of  agents, with much simpler operations to induce permutation invariance, and being much more parameter efficient by doing so.

To scale up learning in MARL problem in the presence of a large number, even infinitely many agents, 
mean-field approximation has been explored to directly model the population behavior of the agents. 
\citet{yang2017learning} consider a mean-field game with  deterministic linear state transitions, and show it can be reformulated as a mean-field MDP, with mean-field state residing in a finite-dimensional probability simplex.
 \citet{yang2018mean} take the mean-field approximation over actions, such that the interaction for any given agent and the population is approximated by the interaction between the agent's action and the averaged actions of its neighboring agents. 
 However, the motivation for averaging over local actions remains unclear, and it generally requires a sparse graph over agents. In practice, properly identifying such structure also demands extensive prior knowledge.
  \citet{carmona2019model} motivate a mean-field MDP from the perspective of mean-field control.
 The mean-field state therein lies in a probability simplex and thus continuous in nature. To enable the ensuing Q-learning algorithm, discretization of the joint state-action space is necessary.
\citet{wang2020breaking} motivate a mean-field MDP from permutation invariance, but assume a central controller coordinating the actions of all the agents, and hence is restricted to handling the curse of many agents from the exponential blowup of joint state space. 
 In contrast, our formulation of mean-field approximation allows agents to make their own local actions without resorting to  a centralized controller.

We propose a mean-field Markov decision process motivated from the permutation invariance structure of  cooperative MARL, which can be viewed as a natural limit of finite-agent MDP by taking the number of agents to infinity. 
Such a mean-field MDP generalizes traditional MDP, with each state representing a distribution over the state space of a single agent.
The mean-field MDP provides us a tractable formulation to model MDP with many agents, even  infinite number of agents.
We further propose the Mean-Field Proximal Policy Optimization (MF-PPO) algorithm, at the core of which is a pair of permutation invariant actor and critic neural networks. 
These networks are implemented based on DeepSet \citep{zaheer2017deep}, which uses convolutional type operations to induce permutation invariance over the set of inputs.
We show that with sufficiently many agents, MF-PPO converges to the optimal policy of the mean-field MDP with a sublinear sample complexity independent of the number of agents.
To support our theory,  we conduct numerical experiments on the benchmark multi-agent particle environment (MPE) and show that our proposed method requires a smaller number of model parameters and attains a better performance compared to multiple baselines.


\textbf{Notations}.
For a set $X$, we denote $\cP(X)$ as the set of distribution on $X$. $\delta_x$ denotes the Dirac measure supported at $x$. 
For $\mathbf{s} = (s_1, \ldots, s_N)$, we use $\mathbf{s}   \iid p$ to denote that each $s_i$ is independently sampled from distribution $p$.
For a function $f:X\to \RR$ and a distribution $\pi \in \cP(X)$, we write $\inner{f}{\pi} = \EE_{a \sim \pi} f(a)$. We write $[m] $ in short for $\{1,\ldots,m\}$.


\vspace{-0.1in}
\section{Problem Setup}\label{sec:background}
\vspace{-0.05in}
%

We focus on studying multi-agent systems with cooperative, homogeneous agents, where the agents within the system are of similar nature and hence cannot be distinguished from each other. 
Specifically, we consider a discrete time control problem with $N$ agents, formulated as a Markov decision process $(\cS^N, \cA^N, \PP, \mathrm{r})$.
We define the joint state space $\cS^N$  to be  the Cartesian product of the finite state space $\cS$ for each agent, and similarly define joint action space $\cA^N$.
The homogeneous nature of the system is reflected in the transition kernel $\PP$ and the shared reward $\mathrm{r}$, which satisfies:
\begin{align}
    \mathrm{r}(\mathbf{s}_t, \mathbf{a}_t)  = \mathrm{r} \rbr{\kappa(\mathbf{s}_t), \kappa({\mathbf{a}_t})}, ~~~
    \PP(\mathbf{s}_{t+1}| \mathbf{s}_t,  {\mathbf{a}_t})  = \PP\rbr{\kappa(\mathbf{s}_{t+1})| \kappa(\mathbf{s}_t),  \kappa({\mathbf{a}_t})} \label{eq:system},
\end{align}
for all $ (\sbb_t,  \ab_t) \in  \cS^N \times \cA^N$ and permutation mapping $\kappa(\cdot)$.
In other words, it is the configuration, rather than the identities of the agents, that affects the team reward, and the transition to the next configuration solely depends on the current configuration. See \Cref{fig:system} for a detailed illustration.
 Such permutation invariance can be found in many real-world scenarios, such as distributed control of multiple autonomous vehicles, team sports and social economic systems \citep{zheng2020ai, cao2013overview,kalyanakrishnan2006half}.


\begin{figure}[htb!]
    \centering
    \includegraphics[width=0.5\linewidth]{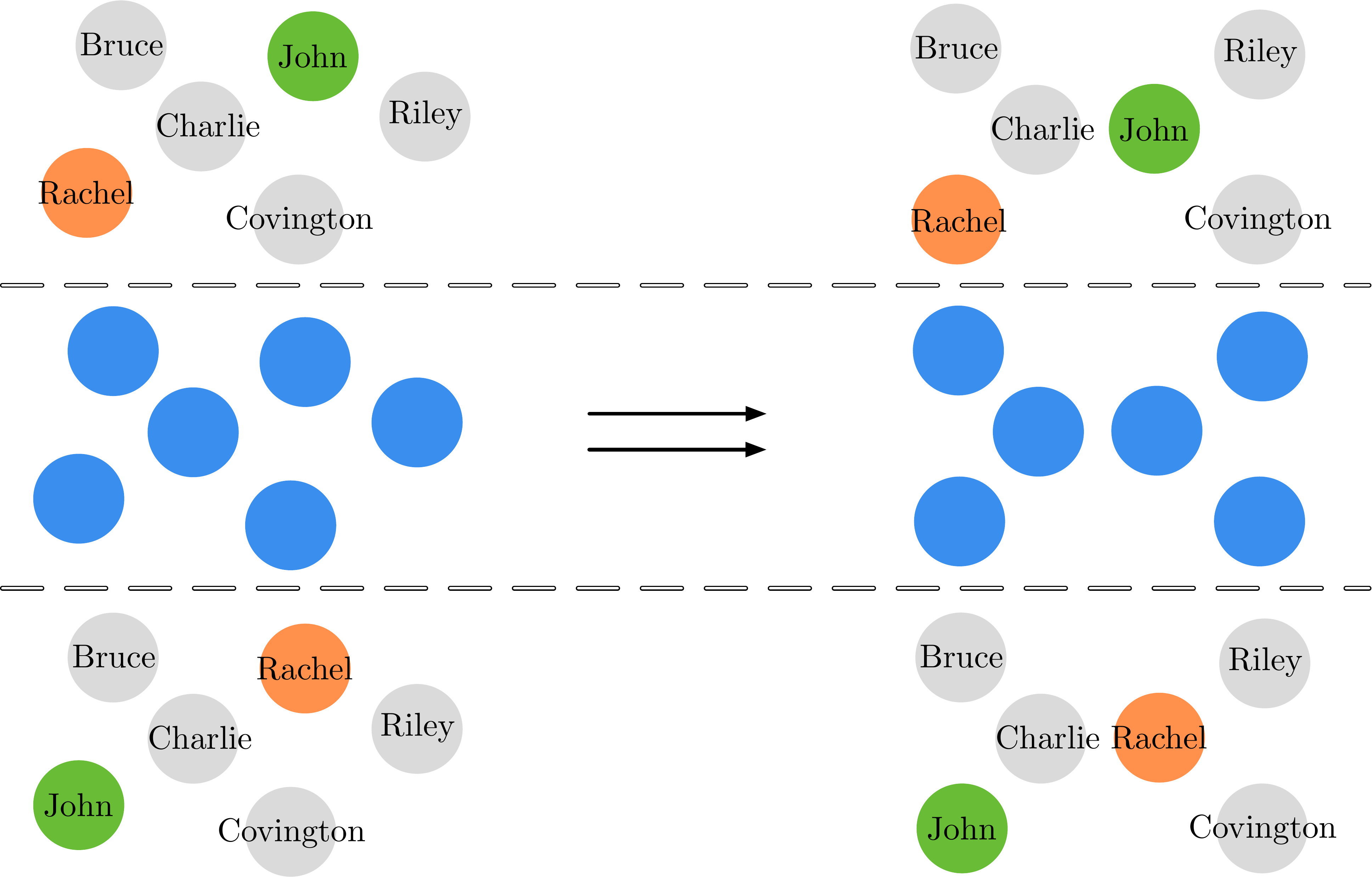}
    \caption{Illustration of permutation invariance. Exchanging identities of the agents (first and third row) does not change the transition of the underlying system (second row).}
    \label{fig:system}
\end{figure}

Our goal is to find the optimal policy $\nu$ within a policy class which maximizes the  expected discounted reward 
$$V^{\nu}(\mathbf{s})\!=\!(1-\gamma) \EE \cbr{\sum_{t=0}^\infty \gamma^t \rmm(\mathbf{s}_t, \ab_t)|\mathbf{s}_0\!=\!\mathbf{s}, \ab_t\!\sim\!\nu(\sbb_t), \forall t\!\geq\!0}.$$
Our first result shows that learning with permutation invariance advocates permutation invariant network design.
\begin{proposition}\label{prop:invariant_policy}
For cooperative MARL satisfying \eqref{eq:system}, there exists an optimal policy $\nu^*$ that is permutation invariant, i.e.,  $\nu^*(\sbb, \ab) = \nu^*\rbr{\kappa(\sbb), \kappa(\ab)}$ for any permutation mapping $\kappa(\cdot)$.
In addition, for any permutation invariant policy $\nu$, the value function $V(\cdot)$ and the state-action value function $Q(\cdot)$ is also permutation invariant
\begin{align*}
V^\nu(\sbb) = V^\nu\rbr{\kappa(\sbb)}, ~~ Q^\nu(\sbb, \ab) = Q^\nu \rbr{\kappa(\sbb), \kappa(\ab)},
\end{align*}
where $Q^\nu(\sbb, \ab) = \EE_{\sbb'} \cbr{r(\sbb, \ab) + \gamma V^\nu(\sbb') }$.
\end{proposition}

Proposition \ref{prop:invariant_policy} has an important implication for architecture design, as it states that it suffices to search within the permutation invariant policy class and value function class. To the best of our knowledge, this is the first theoretical justification of permutation invariant network design for learning with homogeneous agents.

We focus on the factorized policy class with parameter sharing scheme, where each agent makes its own decision given local state. 
Specifically, the joint policy $\nu$ can be factorized as $\nu(\ab|\sbb) = \prod_{i=1}^N \mu(a_i|s_i)$, where $\mu(\cdot)$ denotes the shared local mapping.
Such a policy class is widely adopted in the celebrated centralized training -- decentralized execution paradigm \citep{lowe2017multi},  due to its light overhead in the deployment phase and favorable performances.
However, directly learning such factorized policy remains challenging, as  each agent needs to estimate its state-action value function, denoted as $Q^\nu(\sbb, \ab)$. The search space during learning is $(|\cS| \times |\cA|)^N$, which scales exponentially with respect to the number of agents. The large search space poses as a significant roadblock for efficient learning, and was recently coined as the curse of many agents \citep{menda2018deep}.

To address the curse of many agents, we exploit the homogeneity of the system and take mean-field approximation. We begin by taking the perspective of agent $i$, which is arbitrarily chosen from the $N$ agents. 
We denote the state of such agent by $s$ and the states of the rest of the agents by $\sbb_{r}$.
One can verify that when permuting the state of all the other agents, the value function remains unchanged; additionally, we can further characterize the value function as a function of  the local state, and the empirical state distribution over the rest of agents.

\begin{proposition}\label{prop:val_func}
For any permutation mapping $\kappa(\cdot)$,
the value function satisfies 
$
  V^{\nu}(s, \sbb_r)  = V^{\nu} (s, \kappa(\sbb_r)).
$
Additionally, there exists $g_\nu$ such that:
\begin{align*}
 V^{\nu}(s, \sbb_r) = g_\nu(s, \hat{\pmm}_{\sbb_r}),
\end{align*}
where $ \hat{\pmm}_{\sbb_r} = \frac{1}{N} \sum_{s \in \sbb_r} \delta_{s}$ is the empirical state distribution over the rest of the agents. 
\end{proposition}
For a system with a large number of agents (e.g., financial markets, social networks), the empirical state distribution can be seen as the concrete realization of  the underlying distribution of the population.
Motivated from this observation and Proposition \ref{prop:val_func}, we formulate the following mean-field MDP that can be seen as the limit of finite-agent MDP in the presence of infinitely many homogeneous agents. 

\begin{definition}[mean-field MDP]
The mean-field MDP consists of elements of the following: state $(s, \mathrm{d}_\cS) \in \overline{\cS} \subseteq \cS \times \cP(\cS)$; action $\overline{a} \in \overline{\cA} \subseteq \cA^{\cS}$; 
reward $\rmm(s, \mathrm{d}_\cS, \overline{a})$; transition kernel $\PP(s' ,\mathrm{d}_\cS' | s, \mathrm{d}_\cS, \overline{a}$).
\end{definition}
The defined  mean-field MDP has an intimate connection with our previously discussed finite-agent MDP: here $(s, \mathrm{d}_\cS)$ represents the joint state viewed from the individual agent with state $s$; 
each agent   makes a local decision using the shared local mapping $\oa:\cS \to \cA$  and receive a local reward, whose cumulative effect  is captured by the team reward $\rmm(s, \mathrm{d}_\cS, \overline{a})$; finally, the joint state transits to $(s' ,\mathrm{d}_\cS')$ according to kernel $\PP(s' ,\mathrm{d}_\cS' | s, \mathrm{d}_\cS, \overline{a})$. 
We can clearly see that the mean-field MDP has a step-to-step correspondence with the finite-agent MDP with homogeneous agents, thus making it a natural generalization in the presence of large number of agents.  

Our goal is to learn efficiently a policy $\pi$ whose expected discounted reward is maximized. To facilitate ensuing discussions, we define the value function as
\begin{align*}
    V^{\pi}(s, \mathrm{d}_\cS ) = (1-\gamma) \EE \cbr{\sum_{t=0}^\infty \gamma^t \rmm(s_t, \mathrm{d}_{\cS,t} , \overline{a}_t)}, 
    \end{align*}
   where $(s_0, \mathrm{d}_{\cS,0}) = (s, \mathrm{d}_\cS)  , \overline{a}_t \sim \pi(s_t, \mathrm{d}_{\cS,t}), \forall t\geq 0$, and Q-function as 
   \begin{align*}
      Q^{\pi}(s, \mathrm{d}_\cS , \overline{a}) = (1-\gamma) \EE \cbr{\sum_{t=0}^\infty \gamma^t \rmm(s_t, \mathrm{d}_{\cS,t}, \overline{a}_t )},
\end{align*}
where $(s_0, \mathrm{d}_{\cS,0}) = (s, \mathrm{d}_\cS)  , \overline{a}_0 = \oa,  \overline{a}_t \sim \pi(s_t, \mathrm{d}_{\cS,t})$.
The optimal policy is then denoted by $$\pi^* \in \argmax  V^{\pi}(s, \mathrm{d}_\cS ).$$ 

Despite the intuitive analogy to finite-agent MDP, solving the mean-field MDP posses some unique challenges. 
In addition to the unknown transition kernel and reward, the mean-field MDP takes a distribution as its state, which we do not have complete information of during training.
In the following section, we propose our mean-field Proximal Policy Optimization (MF-PPO) algorithm that,  with a careful architecture design, can solve such mean-field MDP in a model-free fashion efficiently.


\vspace{-0.1in}
\section{Mean-field Proximal Policy Optimization}\label{sec:algo}
\vspace{-0.05in}

Our algorithm falls into the category of the actor-critic learning paradigm, consisting of alternating iterations of policy evaluation and improvement.
The unique features of MF-PPO lie in the facts: 
(1) it exploits permutation invariance of the system,  reducing search space of the actor/critic networks drastically and enables much more efficient learning; 
(2) it can handle a varying number of agents. For simplicity of exposition, we consider a fixed number of agents here.

Throughout the rest of the section, we assume that for any joint state  $(s, \mathrm{d}_{\cS}) \in \cS \times \cP(\cS)$, the agent has access to $N$ $\mathrm{i.i.d.}$ samples $\{s_i\}_{i=1}^N$ from $\mathrm{d}_{\cS}$. We denote concatenation of such samples as $\sbb \in \cS^N$ and write $\mathbf{s} \iid \mathrm{d}_{\cS}$.

MF-PPO maintains a pair of actor (denoted by $F^A$) and critic networks (denoted by $F^Q$), and uses the actor network to induce an energy-based policy $\pi(\overline{a}|s, \mathrm{d}_{\cS})$.
Specifically, given state $(s, \mathrm{d}_{\cS})$, the actor network induces a distribution on the action set $\overline{\cA}$ according to 
$$\pi(\overline{a}|s, \mathrm{d}_{\cS}) \propto \exp \cbr{\tau^{-1} F^A(s, \mathrm{d}_{\cS}, \overline{a})},$$
where $\tau$ denotes the temperature parameter.
We use $\pi \propto \exp \cbr{F^A}$ to denote the dependency of the policy on  the energy function.

\noindent \textbf{$\bullet$} \textbf{Permutation-invariant Actor and Critic}.
We adopt a permutation invariant design of the actor and critic network.
Specifically, given individual state $s \in \cS$ and sampled states $\sbb   \iid \mathrm{d}_{\cS}$, 
the actor (resp. critic) network $F^A$ (resp. $F^Q$) satisfies $$F^A(s, \sbb, \overline{a}) = F^A(s, \kappa(\sbb), \overline{a})$$ for any permutation mapping $\kappa$.
With permutation invariance, the search space of the actor/critic network polynomially depends on the number of agents $N$.
\begin{proposition}\label{prop:search_space}
The search space of a permutation invariance actor (critic) network is at the order of $$\rbr{\sum_{k=1}^{\min\{|\cS|, N\}} \binom{N-1}{k-1} \binom{|\cS|}{k}}|\cS| |\overline{\cA}|;$$ Additionally, if $|\cS| < N$, then the search space depends on $N$ at the order of $N^{|\cS|}$.
\end{proposition}
 Compared to architectures without permutation invariance, whose search space depends on $N$ at the order of $(|\cS| |\cA|)^N $, we can clearly see the search space of MF-PPO can be exponentially smaller.
 
Motivated by the characterization of the permutation invariant set function in \citet{zaheer2017deep},  the actor/critic network in MF-PPO takes the form of Deep Sets architecture, i.e.,  $$F^A(s, \sbb, \oa)  = h\rbr{\sum_{s'\in \sbb}\phi(s,s',  \overline{a})/N}.$$ 
Both  networks first aggregate local information by averaging over the output of a shared sub-network among agents, before feeding the aggregated information into a subsequent network.
See Figure \ref{fig:deepsets} for a detailed illustration. 
Effectively, by the average pooling layer and the preceding parameter sharing scheme, the network can keep its output unchanged when permuting the ordering of agents.
Compared to a Graph Convolutional Neural Network \citep{kipf2016semi}, the averaging operations are well suited for homogeneous agents and more parameter-efficient.

Naturally, the actor network when given the joint state-action pair $(s, \mathrm{d}_{\cS}, \oa) $ is given by $$F^A(s, \mathrm{d}_{\cS}, \oa)  = h\rbr{ \EE_{s' \sim \mathrm{d}_{\cS} }\phi(s,s',  \overline{a})}.$$
We assume $F^A$ is parameterized by a neural network with parameters $\alpha \in \RR^D$, which is to be learned during training.
 We let the function class of all possible actor networks be denoted by $\cF^A$.
 This same architecture design applies to the critic network, with learnable parameters denoted by $\theta \in \RR^D$ and the function class denoted by $\cF^Q$.
MF-PPO then consists of successive iterations of policy evaluation and policy improvement steps described in detail below. 

\begin{figure}[htb!]
    \centering
    \includegraphics[width=0.55\linewidth]{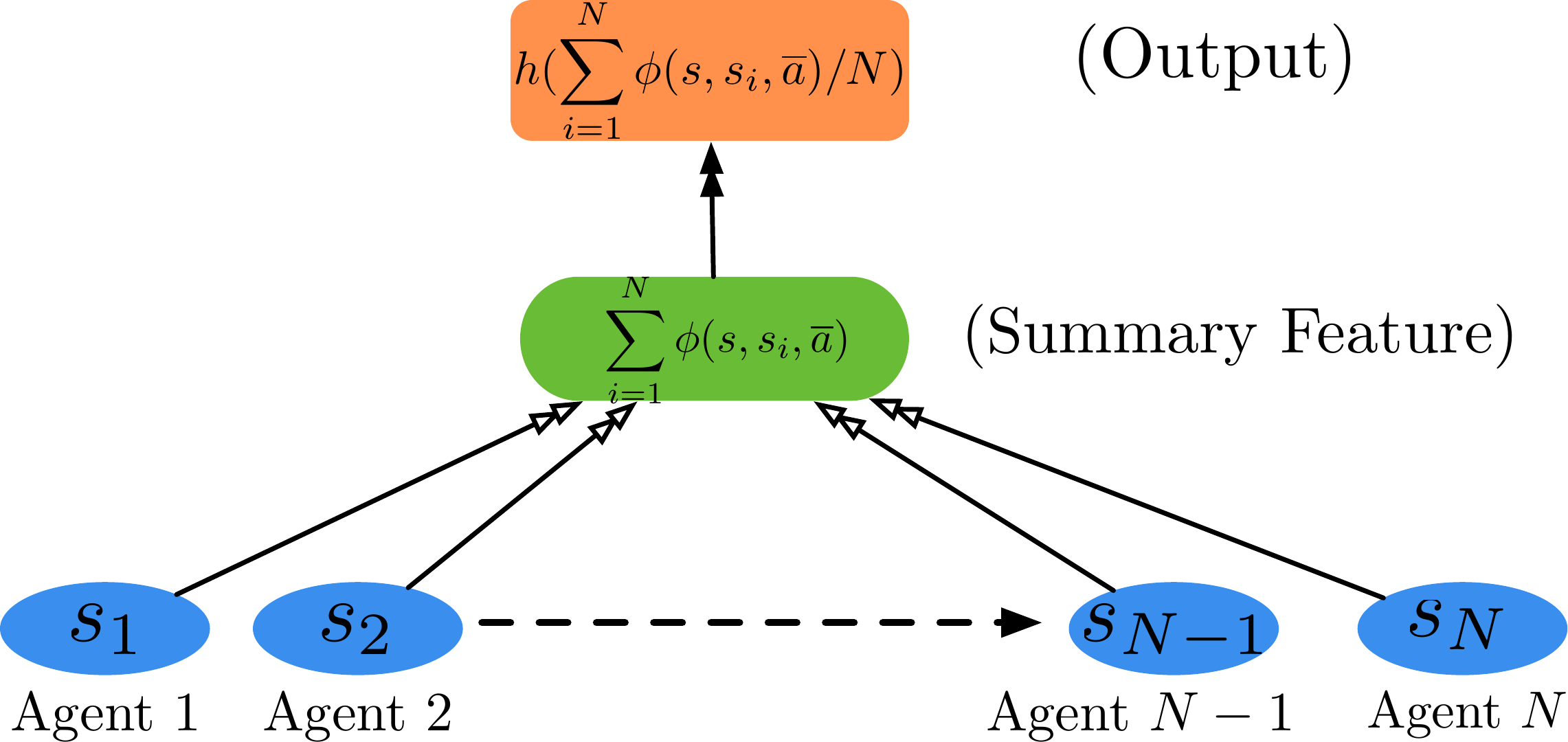}
    \caption{Illustration of a DeepSet parameterized critic network.}
    \label{fig:deepsets}
\end{figure}

 \textbf{$\bullet$}  \textbf{Policy Evaluation}.
At the $k$-th iteration of MF-PPO, we first update the critic network $F^Q$ by minimizing the mean squared Bellman error while holding the actor network $F^{A,k}$ fixed.
We denote the policy induced by the actor network as $\pi_k$, the stationary state distribution of policy $\pi_{k}$ as $\nu_k$, and the stationary state-action distribution as $\sigma_k =  \nu_k \pi_k$. Thus, we follow the update
\begin{align}\label{update:evaluation}
\theta_{k}  =\argmin_{\theta \in \mathbb{B}(\theta_0, R_\theta) } \EE_{\sigma_k} &\bigg\{F^Q_\theta(s, \mathrm{d}_{\cS},\oa) 
- \sbr{\cT^{\pi_k} F^Q_{\theta}}(s, \mathrm{d}_{\cS},\oa) \bigg\}^2,
\end{align}
where $ \mathbb{B}(\theta_0, R_\theta)$ denotes the Euclidean ball with radius $R_\theta$ centered at the initialized parameter $\theta_0$, and the Bellman evaluation operator $\cT^\pi$ is given by 
\begin{align*}
\cT^\pi \cF^Q_\theta(s, \mathrm{d}_{\cS},\oa) = 
\EE \cbr{ (1-\gamma) r(s, \mathrm{d}_{\cS},\oa) +  \gamma F^Q_\theta(s', \mathrm{d}'_{\cS},\oa') },
\end{align*}
where $\mathrm{d}'_{\cS} \sim \PP(\cdot|\mathrm{d}_{\cS}, \oa), \oa' \sim \pi(\cdot | s', \mathrm{d}'_{\cS})$.
We solve \eqref{update:evaluation} by T-step temporal-difference (TD) update and  output $\theta_k = \theta(T)$.
At the $t$-th iteration of the TD update, 
\begin{align*}
\theta(t+1/2)  &  = \theta(t) -  \eta \bigg\{ F^Q_{\theta(t)}(s, \sbb, \oa) - (1-\gamma) r(s, \mathrm{d}_{\cS},\oa) 
-   \gamma F^Q_{\theta(t)}(s', \sbb', \oa')\bigg\} \nabla_\theta F^Q_{\theta(t)}(s, \sbb, \oa) , \\
\theta(t+1) & = \Pi_{ \mathbb{B}(\theta_0, R_\theta) } \rbr{ \theta(t+1/2)  }  ,
\end{align*}
where we sample
\begin{align*}
(s, \mathrm{d}_{\cS}, \oa) \sim \sigma_k,~(s', \mathrm{d}'_{\cS}) \sim \PP(\cdot| s, \mathrm{d}_{\cS}, \oa), ~~~
\oa' \sim \pi_k(\cdot|s', \mathrm{d}'_{\cS}),~\mathbf{s} \iid \mathrm{d}_{\cS},~\mathbf{s}' \iid \mathrm{d}'_{\cS}.
\end{align*}
We use $\Pi_{\mathrm{X}}(\cdot)$ to denote the orthogonal projection onto set $\mathrm{X}$,  and $\eta$ to denote the step size.
The details are summarized in Algorithm \ref{alg:evaluation} in Appendix \ref{append:alg}.

 \textbf{$\bullet$}  \textbf{Policy Improvement}. Following the policy evaluation step, MF-PPO updates its policy  by updating the policy network $F^A$, which is the energy function associated with the policy.
We update the policy network by
\begin{align*}
 \pi_{k+1}  = \argmax_{\substack{\pi \propto \exp \cbr{F^{A,k+1}}, F^{A,k+1} \in \cF^A}} \EE_{\nu_k} \bigg\{  \inner{F^{Q}_{\theta_k}(s, \mathrm{d}_{\cS}, \cdot)}{\pi(\cdot | s, \mathrm{d}_{\cS})} - \upsilon_k \mathrm{KL} \rbr{ \pi(\cdot|s, \mathrm{d}_{\cS}) \big\| \pi_{k}(\cdot|s, \mathrm{d}_{\cS})} \bigg\}.
\end{align*}
The update rule intuitively reads as increasing the probability for choosing action $\oa$ if it yields a higher value for critic network $F^Q(s, \mathrm{d}_{\cS}, \oa)$, which can be viewed as a softened version of policy iteration \citep{bertsekas2011approximate}. 
Additionally, by controlling the proximity parameter $\upsilon_k$, we can control the softness of the update, with $\upsilon \to 0$ yielding the vanilla policy iteration.
Moreover, without constraint  of $F^{A,k+1} \in \cF^A$, such an update would have a nice closed form expression, and  $\pi_{k+1}$ itself is another energy-based policy.

\begin{proposition}\label{update:poly_regression}
Let $\pi_{k} \propto \exp \rbr{\tau_k^{-1} F^{A}_{\alpha_k}}$ denote the energy-based policy, then the update
\begin{align*}
\overline{\pi}_{k+1}  = \argmax_\pi  \EE_{\nu_k} \bigg\{  \inner{F^{Q}_{\theta_k}(s, \mathrm{d}_{\cS}, \cdot)}{\pi(\cdot | s, \mathrm{d}_{\cS})}  - \upsilon_k \mathrm{KL} \rbr{ \pi(\cdot|s, \mathrm{d}_{\cS}) \big\| \pi_{k}(\cdot|s, \mathrm{d}_{\cS})} \bigg\}
\end{align*}
yields $\overline{\pi}_{k+1} \propto \exp \{ \upsilon_k^{-1}F^{Q}_{\theta_k}+ \tau_k^{-1} F^{A}_{\alpha_k} \}$.
\end{proposition}

To take into account that the representable function of actor network resides in $\cF_A$, we update the policy by projecting the energy function of  $\overline{\pi}_{k+1}$ back to $\cF_A$.
Specifically, by denoting $$\pi_{k+1} \propto \exp \{ \tau_{k+1}^{-1} F^{A}_{\alpha_{k+1}} \},$$ we recover the next actor network $F^{A}_{\alpha_{k+1}}$ (i.e., energy) by 
performing the following regression task
\begin{align}\label{update:improvement}
\alpha^*_{k+1}  = \argmin_{\alpha \in  \mathbb{B}(R_\alpha, \alpha_0)  } \EE_{\tilde{\sigma}_k} \bigg\{ F^A_{ \alpha}(s, \mathrm{d}_{\cS}, \oa)  - \tau_{k+1} \sbr{ \upsilon_k^{-1} F^{Q}_{\theta_k}(s, \mathrm{d}_{\cS}, \oa) + \tau_k^{-1} F^A_{ \alpha_k}(s, \mathrm{d}_{\cS}, \oa)} \bigg\}^2,
\end{align}
where $\tilde{\sigma}_k = \nu_k \pi_0$.
We approximately solve  \eqref{update:improvement} via T-step stochastic gradient descent (SGD), and output  $$\alpha_{k+1} = \alpha(T) \approx \alpha_{k+1}^*.$$ 
At the $t$-th iteration of SGD,
\begin{align*}
 \alpha(t+1/2) & = \alpha(t) - \eta \nabla_\alpha F^A_{ \alpha(t)}(s, \sbb, \oa) \bigg\{F^A_{\alpha(t)}(s, \mathbf{s}, \oa) - \tau_{k+1} \rbr{ \upsilon_k^{-1} F^{Q}_{\theta_k}(s, \sbb, \oa) + \tau_k^{-1} F^A_{\alpha_k}(s, \sbb, \oa)}\bigg\} ,  \\
  \alpha(t+1) & = \Pi_{ \mathbb{B}(R_\alpha, \alpha_0)}  \rbr{ \alpha(t+1/2) } ,
\end{align*}
where we sample $(s, \mathrm{d}_{\cS}, \oa) \sim \tilde{\sigma}_k$, and $\sbb \iid \mathrm{d}_{\cS}$, and $\eta$ is the step size.
The details are summarized in Algorithm \ref{alg:improvement} of Appendix \ref{append:alg}.
Finally, we present the complete MF-PPO in Algorithm \ref{alg:m_ppo}.

\begin{algorithm}[htb!]
    \caption{Mean-Field Proximal Policy Optimization }
    \label{alg:m_ppo}
    \begin{algorithmic}
    \STATE{\textbf{Require}: Mean-field MDP $(\overline{\cS}, \overline{\cA}, \PP,  r)$, discount factor $\gamma$; number of outer iterations $K$, number of inner updates $T$;
    policy update parameter $\upsilon$, step size $\eta$,  projection radius $R_\alpha,R_\theta$.}
    \STATE{\textbf{Initialize:} $\tau_0 \leftarrow 1$, $F^{A,0} \leftarrow 0$, $\pi_{0} \propto \exp \{\tau_0^{-1} F^{A,0}\}$} (uniform policy).
    \FOR{$k=0, \ldots, K-1$}
    \STATE{Set temperature parameter $\tau_{k+1} \leftarrow \upsilon \sqrt{K}/(k+1)$, and proximity parameter $\upsilon_k \leftarrow \upsilon \sqrt{K}$  }
    \STATE{Solve \eqref{update:evaluation} to update the critic network $F^Q_{\theta_k}$, using TD update (Algorithm \ref{alg:evaluation})}
    \STATE{Solve \eqref{update:improvement} to update the actor network for $F^A_{\alpha_{k+1}}$, using SGD update (Algorithm \ref{alg:improvement})}
    \STATE{Update policy: $\pi_{k+1} \propto \exp\{\tau_{k+1}^{-1} F^A_{\alpha_{k+1}}\}$}
    \ENDFOR
    \end{algorithmic}
\end{algorithm}

\vspace{-0.1in}
\section{Global Convergence of MF-PPO}\label{sec:theory}
\vspace{-0.05in}

We present the global convergence of MF-PPO algorithm for two-layer permutation-invariant parameterization of actor and critic network. 
We remark that our analysis can be extended to multi-layer permutation-invariant networks, and we present the two-layer case here for simplicity of exposition. 
Specifically, the actor and critic networks take the form
\begin{align*}
F^A_{\alpha}(s, \sbb, \oa) &= \frac{1}{\sqrt{m}N} \sum_{j=1}^{m_A} \sum_{s' \in \sbb} u_j \sigma \rbr{\alpha_j^\top (s, s', \oa)}, \\
F^Q_{\theta}(s, \sbb, \oa) & = \frac{1}{\sqrt{m}N} \sum_{j=1}^{m_Q} \sum_{s' \in \sbb} v_j \sigma \rbr{\theta_j^\top (s, s', \oa)},
\end{align*} 
where $m_A$ (resp. $m_Q$) is the width of the actor (resp. critic) network, and $\sigma(x) = \max \{x, 0\}$ denotes the ReLU activation. 
We randomly initialize $u_j$ (resp. $v_j$) and first layer weights  $\alpha_0 = [\alpha_{0,1}^\top, \ldots, \alpha_{0, m_A}^\top]^\top \in \RR^{d \cdot m_A}$ (resp. $\theta_0 \in \RR^{d \cdot m_Q}$) by
\begin{align*}
u_j \sim \mathrm{Unif} \cbr{-1, +1}, \alpha_{0, j} \sim \cN(0, \mathrm{I}_d/d), ~~\forall j \in [m].
\end{align*}
For the ease of analysis, we take $m = m_A=m_Q$ and share the initialization of $\alpha_0$ and $\theta_0$ (resp. $u_0$ and $v_0$).
Additionally, we keep $u_j$'s fixed during training, 
and $\alpha$ (resp. $\theta$) within the ball $\BB(\alpha_0, R_A)$ (resp. $\BB(\theta_0, R_Q)$) throughout training.
We define the following function class characterizing the class of previously defined actor/critic network for large network width.
\begin{definition}
Given $R_\beta>0$, define the function class over domain $\cS \times \cS \times \cA$ by
\begin{align*}
\cF_{\beta, m} \!=\!\bigg\{f_\beta(\cdot)  \big \vert  ~f_\beta(s,s', \oa)\!=\!\frac{1}{\sqrt{m}} \sum_{j=1}^m v_{j}  \mathbbm{1}\{\beta_{0,j}^\top (s,s', \oa)\!>\!0\}  \beta_j^\top (s,s', \oa), \norm{\beta- \beta_0}_2 \leq R_\beta\bigg\},
\end{align*}
where $v_j, \beta_{0,j}$ are random weights sampled according to
\begin{align*}
v_j \sim \mathrm{Unif} \cbr{-1, +1}, \beta_{0, j} \sim \cN(0, \mathrm{I}_d/d), ~~\forall j \in [m].
\end{align*}
$\cF_{\beta, m}$ also induces the function class over $\cS \times \cP(\cS) \times \cA$ given by
\begin{align*}
\cF^\cP_{\beta, m}\!=\!\cbr{F(\cdot)~\big \vert~F(s, \mathrm{d}_\cS, \oa)\!=\!\EE_{s' \sim \mathrm{d}_\cS} f_\beta(s, s', \oa), f_\beta\!\in\!\cF_{\beta, m} }.
\end{align*}
\end{definition}
It is well known that functions within $\cF_{\beta,m} $ approximate functions within the reproducing kernel Hilbert space associated with kernel 
$\mathrm{K}(x,y) = \EE_{z \sim \mathcal{N}(0, I_d/d) }\cbr{ \mathbbm{1}(z^\top x>0, z^\top y >0) }$ 
for large network width $m$ 
\citep{jacot2018neural,chizat2018note, cai2019neural, arora2019fine}, whose RKHS norm is bounded by $R_\beta$.
For large $R_\beta$ and $m$, $\cF_{\beta,m} $ represents a rich class of functions.
Additionally, functions within $\cF^\cP_{\beta, m}$ can be viewed as the mean-embedding of the joint state-action pair onto the RKHS space \citep{muandet2016kernel, song2009hilbert, smola2007hilbert}. Below, we make one technical assumption which states  that $\cF^\cP_{\theta, m_Q}$ is rich enough to represent the Q-function of all the policies within our policy class. 

\begin{assumption}\label{assump:richness}
For any policy $\pi$ induced by $F_A \in \cF^A$, we have $Q^\pi \in \cF^\cP_{\theta, m_Q}$.
\end{assumption}

We define two mild conditions stating (1) boundedness of reward, and (2) regularity of the joint state and the policy.
\begin{assumption}\label{assump:mild}
Reward function $r(\cdot) \leq \overline{r}$ for some $\overline{r} >0$.
Additionally, for any $(s, \mathrm{d}_\cS) \in \overline{\cS}$ and any $\pi$,   there exists $c>0$ such that 
$$\EE_{s' \sim \mathrm{d}_\cS, \oa \sim \pi(\cdot|s,\mathrm{d}_\cS)}  \cbr{\indic \rbr{\abs{z^\top (s, s', \oa)} \leq t } }\leq c \cdot \frac{ t}{\norm{z}_2}$$ for any $z \in \RR^d$ and $t > 0$.
\end{assumption}
We remark that the condition in Assumption \ref{assump:mild} holds whenever $\cS$ is bounded and $\mathrm{d}_\cS \times \pi$ has an upper bounded density.

We measure the progress of  MF-PPO in Algorithm \ref{alg:m_ppo} using the expected total reward
\begin{align}\label{objective}
\cL (\pi) = \EE_{\nu^*} \sbr{V^\pi(s, \mathrm{d}_\cS)} = \EE_{\nu^*} \cbr{ \inner{Q^\pi(\cdot|s, \mathrm{d}_\cS)}{\pi(\cdot|s, \mathrm{d}_\cS)}},
\end{align}
where $\nu^*$ is the stationary state distribution of the optimal policy $\pi^*$.
We also denote $\sigma^*$ as the stationary state-action distribution induced by $\pi^*$.
Note that we have: $$\cL(\pi^*) = \EE_{\nu^*} \sbr{ V^{\pi^*}(s, \mathrm{d}_\cS)} \geq \EE_{\nu^*} \sbr{ V^{\pi}(s, \mathrm{d}_\cS)} = \cL(\pi)$$ for any policy $\pi$.
Our main results are presented in the following theorem, showing that $\cL(\pi_k)$ converges to $\cL(\pi^*)$ at a sub-linear rate.

\begin{theorem}[Global Convergence of MF-PPO]\label{thrm:main}
Under Assumptions \ref{assump:richness} and  \ref{assump:mild}, 
the  policies $\{\pi_k\}_{k=1}^K$ generated by Algorithm \ref{alg:m_ppo} satisify
\begin{align*}
\min_{0 \leq k \leq K} \{ \cL(\pi^*) - \cL(\pi_k) \}  \leq \frac{\upsilon\rbr{ \log|\overline{\cA}| + \sum_{k=1}^{K-1} (\varepsilon_k + \varepsilon_k')}}{(1-\gamma) \sqrt{K}}   + \frac{M}{(1-\gamma) \upsilon \sqrt{K}},
\end{align*}
where $M = \EE_{\nu^*} \cbr{ \max_{\oa \in \overline{\cA}} (F^Q_{\theta_0}(s, s',\oa))^2} + 2R_A^2$,  
$\varepsilon_k$ and $\varepsilon_k'$ are defined in Lemma \ref{lemma:err_propagation}.
In particular, suppose at each iteration of MF-PPO, we observe $$N =   \Omega\rbr{K^3 R_A^4 + K R_Q^4 }$$ agents, and the actor/critic network satisfies $$m_A = \Omega \rbr{K^6 R_A^{10} + K^4 R_A^{10} \abs{\overline{\cA}}^2}, m_Q = \Omega \rbr{K^2 R_Q^{10}},$$ and
 $$T = \Omega\rbr{K^3 R_A^4 + K R_Q^4 },$$ then we have
\begin{align*}
\min_{0 \leq k \leq K} \{ \cL(\pi^*) - \cL(\pi_k) \} \leq \frac{\upsilon^2\rbr{ \log|\overline{\cA}| + \cO(1) } + M}{(1-\gamma)\upsilon \sqrt{K}}.
\end{align*}
\end{theorem}

Theorem \ref{thrm:main} states that, given sufficiently many agents and a large enough actor/critic network, MF-PPO attains global optimality at a sublinear rate.
Our result shows that when solving the mean-field MDP, having more agents serves as a blessing instead of a curse. 
In addition, as will be demonstrated in our proof sketch, there exists an inherent phase transition depending on the number of agents, where the final optimality gap is dominated by statistical error for a small number of agents (first phase);  and by optimization error for a large number of agents (second phase). 

The complete proof of Theorem \ref{thrm:main} takes a careful analysis on the error from policy evaluation \eqref{update:evaluation}  and the improvement step \eqref{update:improvement}. The analysis on the outer iterations of MF-PPO can be overviewed as approximate mirror descent, which needs to take into account how the evaluation and improvement error interacts.
Intuitively, the tuple $(\varepsilon_k, \varepsilon'_k)$ describes the the effect of policy update when using approximate policy evaluation and policy improvement, and will be further clarified in the ensuing discussion.
We present here the skeleton of our proof, and defer the technical detail to the appendix.

\textbf{\textit{Proof Sketch.}}
We first establish the global convergence of the policy evaluation \eqref{update:evaluation} and improvement step \eqref{update:improvement}.
\begin{lemma}[Policy Evaluation]\label{thrm:poly_eval}
Under the same assumptions in Theorem \ref{thrm:main}, 
let $Q^{\pi_k}$ denote the Q-function of policy $\pi_k$, let $$
\epsilon_k = \EE_{\init, \sigma_k} \cbr{F_{\overline{\theta}(T)}^Q (\cdot) - Q^{\pi_k} (\cdot)}^2$$  denote policy evaluation error, where $\overline{\theta}(T)$ is the output of Algorithm \ref{alg:evaluation}, we have
\begin{align*}
\epsilon_k  = \cO \rbr{\frac{R_Q^2}{T^{1/2}} + \frac{R_Q^{5/2}}{m_Q^{1/4}} + \frac{R_Q^2}{N^{1/2}}   + \frac{R_Q^3}{m_Q^{1/2}} }.
\end{align*}
\end{lemma}

\begin{lemma}[Policy Improvement]\label{thrm:poly_improv}
Under the same assumptions in Theorem \ref{thrm:main}, let $$\epsilon_{k+1}'\!=\!\EE_{\init, \tilde{\sigma}_k} \cbr{F^A_{ \overline{\alpha}(T)}(\cdot)\!-\!\tau_{k+1}\rbr{\upsilon_k^{-1} F^Q_{\theta_k}(\cdot)\!+\!\tau_k^{-1} F^A_{ \alpha_k}(\cdot)}}^2$$ denote policy optimization error,  where $\overline{\alpha}(T)$ is the output of  Algorithm \ref{alg:improvement}, we have
\begin{align*}
\epsilon_{k+1}' 
=  \cO \rbr{\frac{R_A^2}{T^{1/2}} + \frac{R_A^{5/2}}{m_A^{1/4}} + \frac{R_A^2}{N^{1/2}}   + \frac{R_A^3}{m_A^{1/2}} }.
\end{align*}
\end{lemma}

 Lemma \ref{thrm:poly_eval} and \ref{thrm:poly_improv} show that despite non-convexity, both the policy evaluation and the policy improvement steps converge approximately to the global optimal solution.
In particular, given networks with large width, for a small number of iterations $T$,  the optimization error $\mathcal{O} \rbr{T^{-1/2}}$ dominates the optimality gap; for a large number of iterations $T$, the statistical error $\cO\rbr{N^{-1/2}}$ dominates the optimality gap.


With Lemma \ref{thrm:poly_eval} and  \ref{thrm:poly_improv}, we  illustrate the main argument for the proof of Theorem \ref{thrm:main}. Let us assume the ideal case when $\epsilon_k = \epsilon'_{k+1} = 0$.
Note that for $\epsilon_k = 0$, we obtain the exact $Q$-function of policy $\pi_k$.
For $\epsilon'_{k+1} = 0$, we obtain the ideal energy-based updated policy defined in Proposition \ref{update:poly_regression}. That is,
\begin{align}\label{poly:ideal}
\pi_{k+1}  =  \argmax_\pi \EE_{ \nu_k} \bigg\{ \inner{Q^{\pi_k}(s, \mathrm{d}_\cS, \cdot)}{\pi(\cdot | s, \mathrm{d}_\cS)} -  \upsilon_k \mathrm{KL} \rbr{ \pi(\cdot|s, \mathrm{d}_\cS) \| \pi_{k}(\cdot|s, \mathrm{d}_\cS)} \bigg\}.
\end{align}
Without function approximation,  problem \eqref{poly:ideal} can be solved by treating each joint state $(s, \mathrm{d}_\cS)$ independently, hence one can apply the well known three-point lemma in mirror descent \citep{chen1993convergence} and obtain that, for all $(s, \mathrm{d}_\cS) \in \cS \times \cP(\cS)$:
\begin{align*}
&\hspace{-0.2in}\inner{Q^{\pi_k}(s, \mathrm{d}_\cS, \cdot)}{\pi^*(\cdot | s, \mathrm{d}_\cS) - \pi_{k}(\cdot | s, \mathrm{d}_\cS)}  \\ \leq &  \upsilon_k \bigg\{  \mathrm{KL} \rbr{ \pi^*(\cdot|s, \mathrm{d}_\cS) \| \pi_{k}(\cdot|s, \mathrm{d}_\cS)}  
- \mathrm{KL} \rbr{ \pi^*(\cdot|s, \mathrm{d}_\cS) \| \pi_{k+1}(\cdot|s, \mathrm{d}_\cS)} -  \mathrm{KL} \rbr{ \pi_{k+1}(\cdot|s, \mathrm{d}_\cS) \| \pi_{k}(\cdot|s, \mathrm{d}_\cS) }\bigg\} \\
& ~~~~~~  + \inner{Q^{\pi_k}(s, \mathrm{d}_\cS, \cdot)}{\pi_{k+1} (\cdot | s, \mathrm{d}_\cS) - \pi_{k}(\cdot | s, \mathrm{d}_\cS)} .
\end{align*}
From Lemma 6.1 in \citet{kakade2002approximately}, the expectation of the left hand side yields exactly $(1-\gamma)\cbr{ \cL(\pi^*) - \cL(\pi_k)}$.
Hence we have
\begin{align*}
&\hspace{-0.1in}(1-\gamma)\cbr{ \cL(\pi^*) - \cL(\pi_k)} \\ 
\leq   &  \upsilon_k \EE_{\nu^*} \bigg\{  \mathrm{KL} \rbr{ \pi^*(\cdot|s, \mathrm{d}_\cS) \| \pi_{k}(\cdot|s, \mathrm{d}_\cS)}  - \mathrm{KL} \rbr{ \pi^*(\cdot|s, \mathrm{d}_\cS) \| \pi_{k+1}(\cdot|s, \mathrm{d}_\cS)} -  \mathrm{KL} \rbr{ \pi_{k+1}(\cdot|s, \mathrm{d}_\cS) \| \pi_{k}(\cdot|s, \mathrm{d}_\cS) }\bigg\} \\
& ~~~~~~  +  \EE_{\nu^*} \inner{Q^{\pi_k}(s, \mathrm{d}_\cS, \cdot)}{\pi_{k+1} (\cdot | s, \mathrm{d}_\cS) - \pi_{k}(\cdot | s, \mathrm{d}_\cS)}  .
\end{align*}
Pinsker's Inequality
\begin{align*}
    \mathrm{KL} \rbr{ \pi_{k+1}(\cdot|s, \mathrm{d}_\cS) \| \pi_{k}(\cdot|s, \mathrm{d}_\cS) } \geq 
\frac{1}{2} \norm{\pi_{k+1} - \pi_k}_1^2,
\end{align*}
combined with the observation $\norm{Q^{\pi_k}(s, \mathrm{d}_\cS, \cdot)}_\infty \leq \overline{r}$, 
 and  basic inequality $-ax^2 + bx \leq b^2/(4a)$ for $a>0$ gives us  
\begin{align}
(1-\gamma)\cbr{ \cL(\pi^*) - \cL(\pi_k)} 
\leq   \upsilon_k  \EE_{\nu^*} \bigg\{  \mathrm{KL} \rbr{ \pi^*(\cdot|s, \mathrm{d}_\cS) \| \pi_{k}(\cdot|s, \mathrm{d}_\cS)} - \mathrm{KL}\rbr{ \pi^*(\cdot|s, \mathrm{d}_\cS) \| \pi_{k+1}(\cdot|s, \mathrm{d}_\cS)} \bigg\}  + \overline{r}^2/(2\upsilon_k) . \label{eq:key_recursion}
\end{align}
By setting $\upsilon_k = \cO(\sqrt{K})$, and telescoping the above inequality from $k= 0$ to $ K-1$, we obtain:
 $$\min_{0 \leq k \leq K-1} \{ \cL(\pi^*) - \cL(\pi_k) \}  = \cO(1/\sqrt{K}).$$ 
Note that the key element in the global convergence of MF-PPO is the recursion defined in \eqref{eq:key_recursion}, which holds whenever we have the exact $Q$-function of current policy and no function approximation is used when updating the next policy.
Now MF-PPO conducts approximate policy evaluation ($\epsilon_k >0$), and after obtaining the approximate $Q$-function,
conducts approximate policy improvement step ($\epsilon'_{k+1} >0$). 
In addition,  the error of  approximating the $Q$-function  introduced in the evaluation step can be further compounded in the improvement step. 
Nevertheless, recursion \eqref{eq:key_recursion} still holds approximately, with additional terms representing the policy evaluation/improvement errors.

 \begin{lemma}[\citet{liu2019neural}]\label{lemma:err_propagation}
Let $\epsilon_k$ (evaluation error) and $\epsilon_{k+1}'$ (improvement error) be defined as in Lemma \ref{thrm:poly_eval} and Lemma \ref{thrm:poly_improv}, respectively.  We have:
\vspace{-0.02in}
\begin{align}
(1- \gamma) \rbr{\cL(\pi^*) - \cL(\pi_k)}
  & \leq    \upsilon_k \EE_{\nu_*} \bigg\{   \mathrm{KL} \rbr{ \pi^*(\cdot|s, \mathrm{d}_\cS) \| \pi_{k}(\cdot|s, \mathrm{d}_\cS)}  - \mathrm{KL} \rbr{ \pi^*(\cdot|s, \mathrm{d}_\cS) \| \pi_{k+1}(\cdot|s, \mathrm{d}_\cS)} \bigg\}   \\
   & ~~~~~~ +  \upsilon_k \left(\varepsilon_k + \varepsilon_k'\right)   + \upsilon_k^{-1} M . \label{ineq:telescopt}
\end{align}
where
\vspace{-0.02in}
\begin{align*}
\varepsilon_k &= \tau_{k+1}^{-1} \epsilon_{k+1}' \phi_{k+1}^* + \upsilon_k^{-1} \epsilon_k \psi_k^*,\\
\varepsilon_k' &= \abs{\cA} \tau_{k+1}^{-2} (\epsilon_{k+1}')^2,\\
M &= \EE_{\nu^*} \cbr{ \max_{\oa \in \overline{\cA}} \sbr{F^Q_{\theta_0}(s, \mathrm{d}_\cS, \oa)}^2} + 2R_A^2.
\end{align*}
In addition, $\phi_k^*$ and $\psi_k^*$ are defined by:
\begin{align*}
\phi_k^* &= \EE_{\tilde{\sigma}_k} \sbr{\abs{\mathrm{d}\pi^* / \mathrm{d}\pi_0 -  \mathrm{d}\pi_{k} / \mathrm{d}\pi_0}^2 }^{1/2}, \\
\psi_k^* &= \EE_{\sigma_k} \sbr{\abs{\mathrm{d}\sigma^* / \mathrm{d}\sigma_k -  \mathrm{d}(\nu^*\times \pi_k) / \mathrm{d}\sigma_k}^2 }^{1/2}.
\end{align*}
\end{lemma}
Finally, from Lemma \ref{lemma:err_propagation}, by telescoping inequality \eqref{ineq:telescopt} from $k=0$ to $K-1$, we complete the proof of Theorem \ref{thrm:main}.

\vspace{-0.1in}
\section{Experiments}\label{sec:experiments}
\vspace{-0.05in}

We perform experiments on the benchmark multi-agent particle environment (MPE) used in prior work \citep{lowe2017multi,mordatch2018emergence,liu2019pic}.
In the \textit{cooperative navigation} task, $N$ agents each with position $x_i \in \Rbb^2$ must move to cover $N$ fixed landmarks at positions $y_i \in \Rbb^2$.
They receive a team reward $R = - \sum_{i=1}^N \min_{j \in [N]} \lVert y_i - x_j \rVert_2;$
In the \textit{cooperative push} task, $N$ agents with position $x_i \in \RR^2$ work together to push a ball $x \in \RR^2$ to a fixed landmark $y \in \RR^2$. 
They receive a team reward $R = - \min_{j \in [N]} \lVert x_j - x \rVert_2 - \norm{x -y}_2.$
Both tasks involve homogeneous agents.

We instantiate MF-PPO by building on the original PPO algorithm \citep{schulman2017proximal} with a centralized permutation invariant value-function critic, represented by a DeepSet \citep{zaheer2017deep} network with two hidden layers.
We use a standard two-layer multi-layer perception (MLP) for the actor network in all algorithms.
One version, labeled as \textbf{MF}, has a centralized actor network that outputs the mean and diagonal covariance of a Gaussian distribution over the joint continuous action space.
A second decentralized version with parameter sharing, labeled as \textbf{MF-PS}, has a decentralized actor network that maps individual observations to parameters of a Gaussian distribution over the individual action space.

\begin{figure}[htb!]
\vspace{-0.1in}
\centering
\begin{subfigure}[t]{0.45\linewidth}
    \centering
    \includegraphics[width=1.0\linewidth]{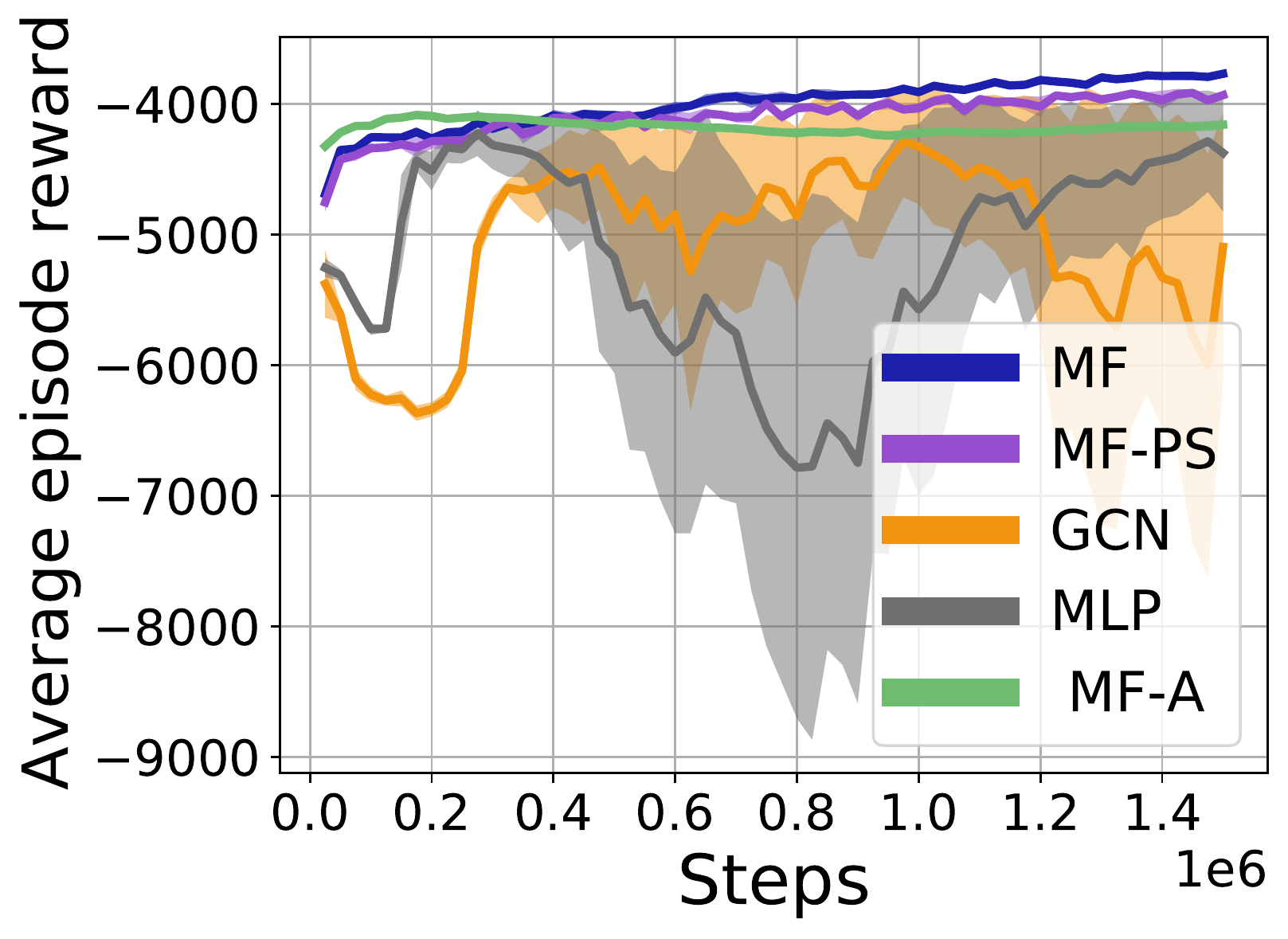}
    \vspace{-0.2in}
    \caption{$N=6$}
    \label{fig:coop_nav_n3}
\end{subfigure}
\hfill
\begin{subfigure}[t]{0.45\linewidth}
    \centering
    \includegraphics[width=1.0\linewidth]{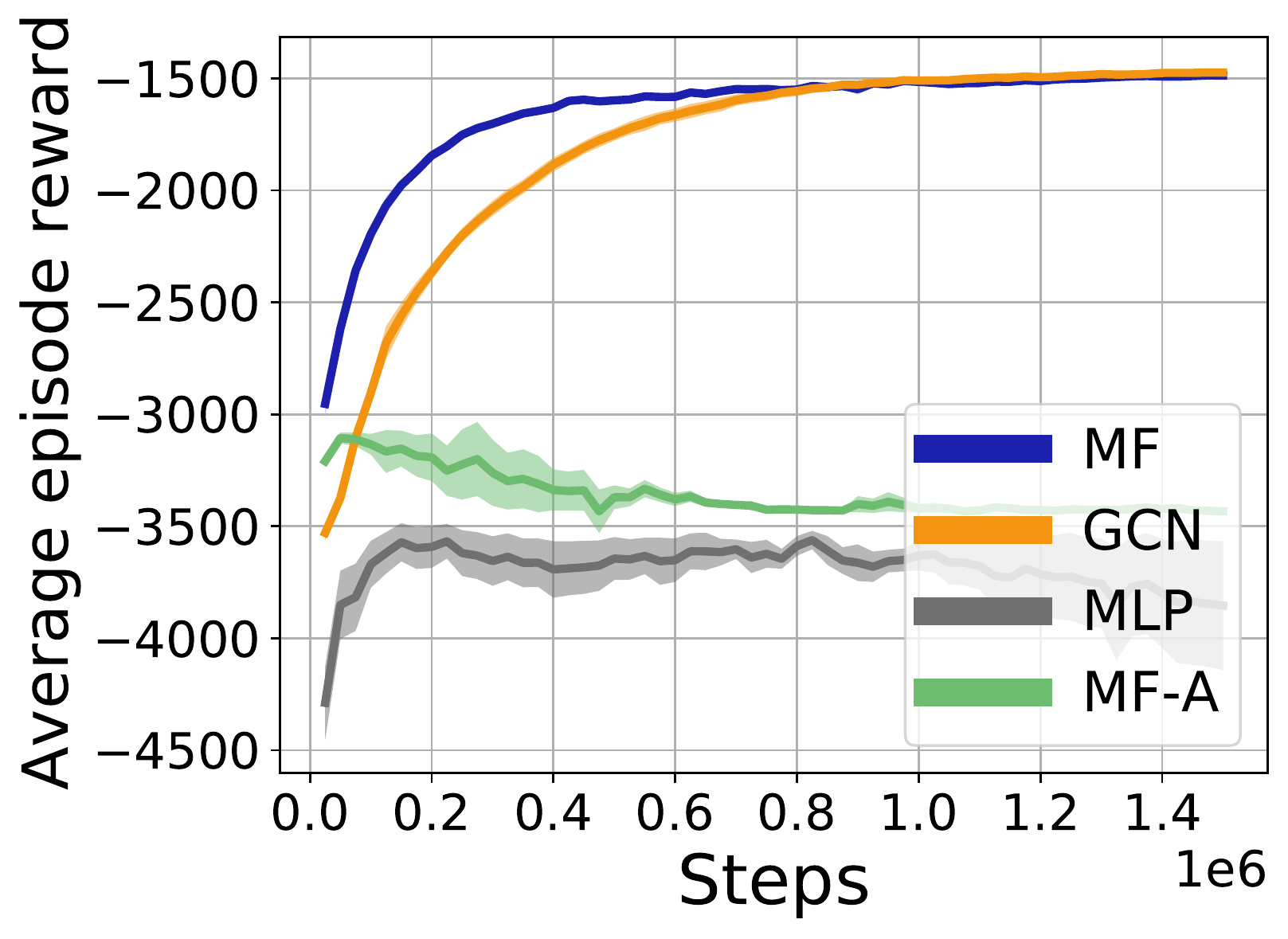}
    \vspace{-0.2in}
    \caption{$N=15$}
    \label{fig:coop_nav_n6}
  \end{subfigure}
\hfill
\begin{subfigure}[t]{0.45\linewidth}
  \centering
  \includegraphics[width=1.0\linewidth]{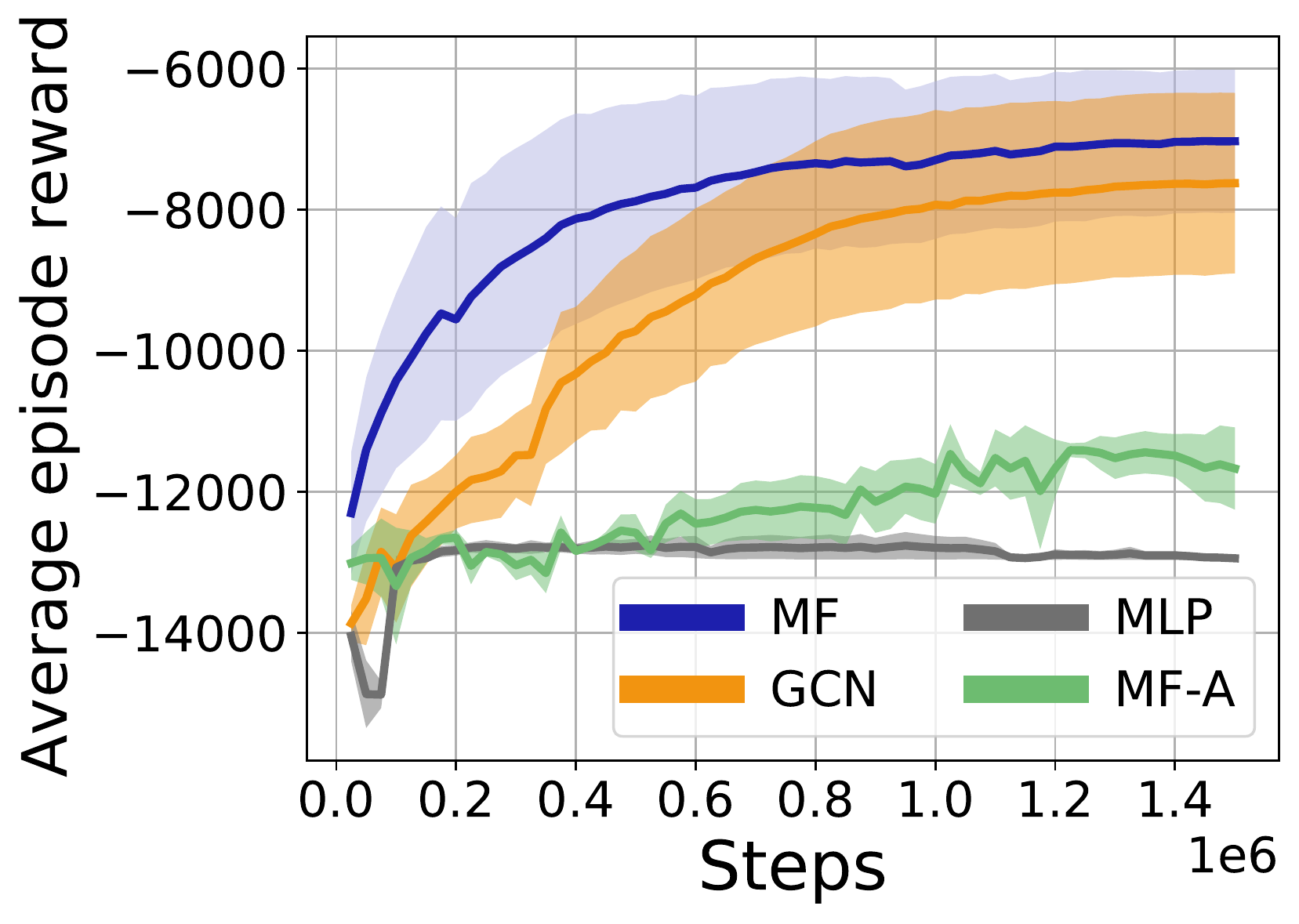}
  \vspace{-0.2in}
  \caption{$N=30$}
  \label{fig:coop_nav_n15}
\end{subfigure}
\hfill
\begin{subfigure}[t]{0.45\linewidth}
  \centering
  \includegraphics[width=1.0\linewidth]{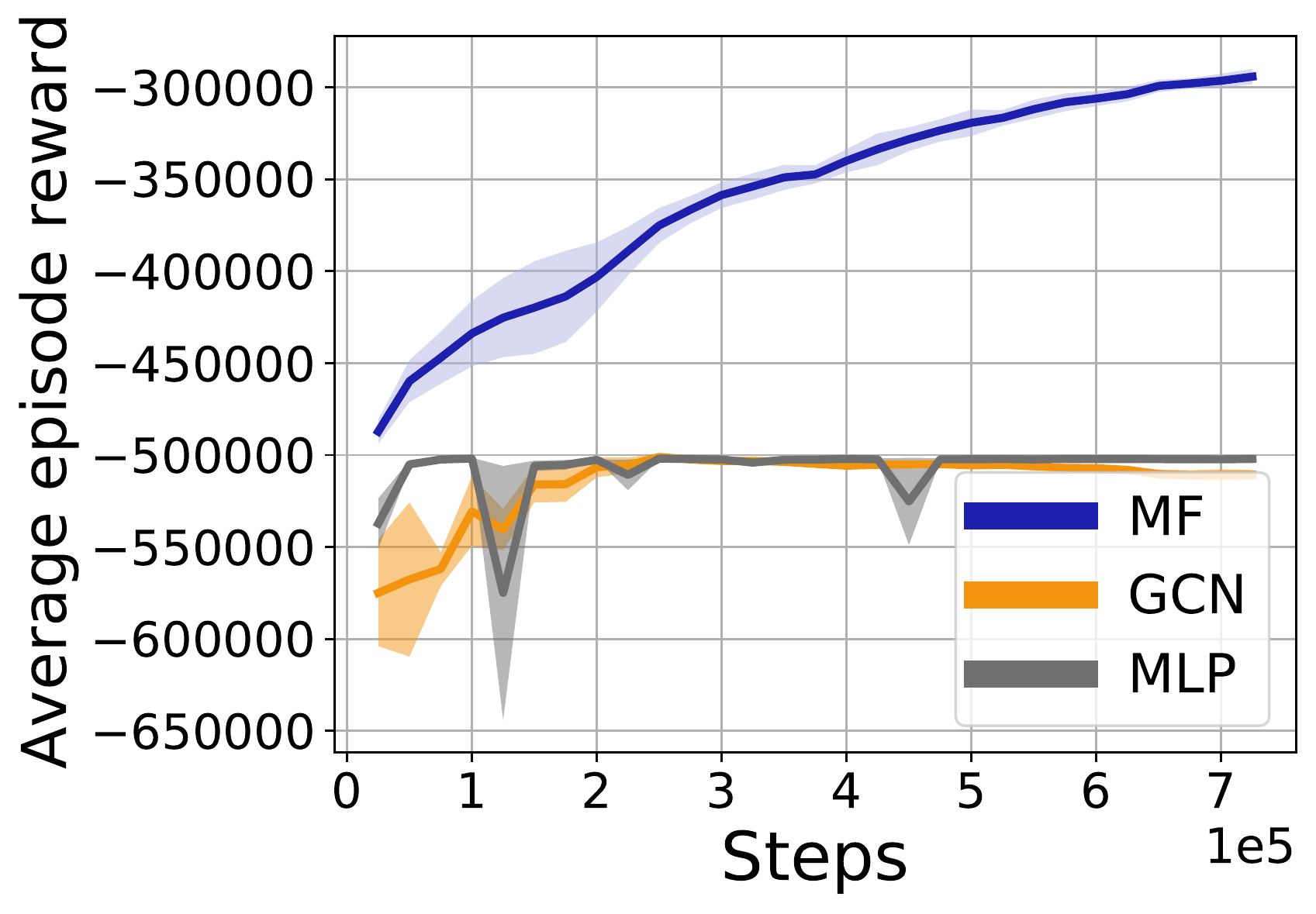}
  \vspace{-0.2in}
  \caption{$N=200$}
  \label{fig:critic_size}
\end{subfigure}
\vspace{-0.05in}
\caption{Performance versus number of environment steps in the multi-agent \textit{cooperative navigation} task, over five independent runs per method. 
Each point is the average reward during 1000 evaluation episodes.
 MF based on DeepSet invariant critic, and built on either PPO or MADDPG, significantly outperforms other critic representations for various number of agents.
 }
\label{fig:results_coop_nav}
\vspace{-0.1in}
\end{figure}

We compare with two other critic representations: one that uses MLP for the centralized critic, labeled \textbf{MLP}, and another that uses a graph convolutional network for the critic \citep{liu2019pic}, labeled \textbf{GCN}.
Note that the GCN representation is permutation invariant if one imposes a fully-connected graph for the agents in the MPE, but this invariance property does not hold for all graphs in general.
We also compare with an extension of \citep{yang2018mean} to the case of continuous action spaces, labeled \textbf{MF-A}, in which each independent DDPG agent $i$ has a decentralized critic $Q(s_i,a_i,\bar{a}_i)$ that takes in the mean of all other agents' actions $\bar{a}_i := \frac{1}{N-1}\sum_{j \neq i} a_j$.
Empirically, as we find that off-policy RL learns faster than on-policy RL in the MPE with higher agent number, regardless of the critic representation, we make all comparisons on top of MADDPG \citep{lowe2017multi} for all $N\geq 15$ cases.
For a fair comparison of all critic representations, we ensure that all neural network architectures contain approximately the same number of trainable weights.

For the cooperative navigation task, 
\Cref{fig:results_coop_nav} shows that the permutation invariant critic representation based on DeepSet enables MF and MF-PS to learn faster or reach a higher performance than all other representations and methods in the MPE with 6, 15, 30 and 200 agents.
\Cref{fig:varying_critic} shows that MF consistently and significantly outperforms GCN as the number of parameters in their critic network varies over a range, with all other settings fixed.
In particular, MF requires much fewer critic parameters to achieve higher performance than GCN.
For the cooperative push task, 
\Cref{fig:coop_push_n15}  demonstrates  similar performance improvement of MF compared to all other methods.
We provide additional experiment detail in Appendix \ref{app:experiment}.


\begin{figure}[htb!]
\centering
\vspace{-0.1in}
\begin{subfigure}[t]{0.48\linewidth}
    \centering
    \includegraphics[width=1.0\linewidth]{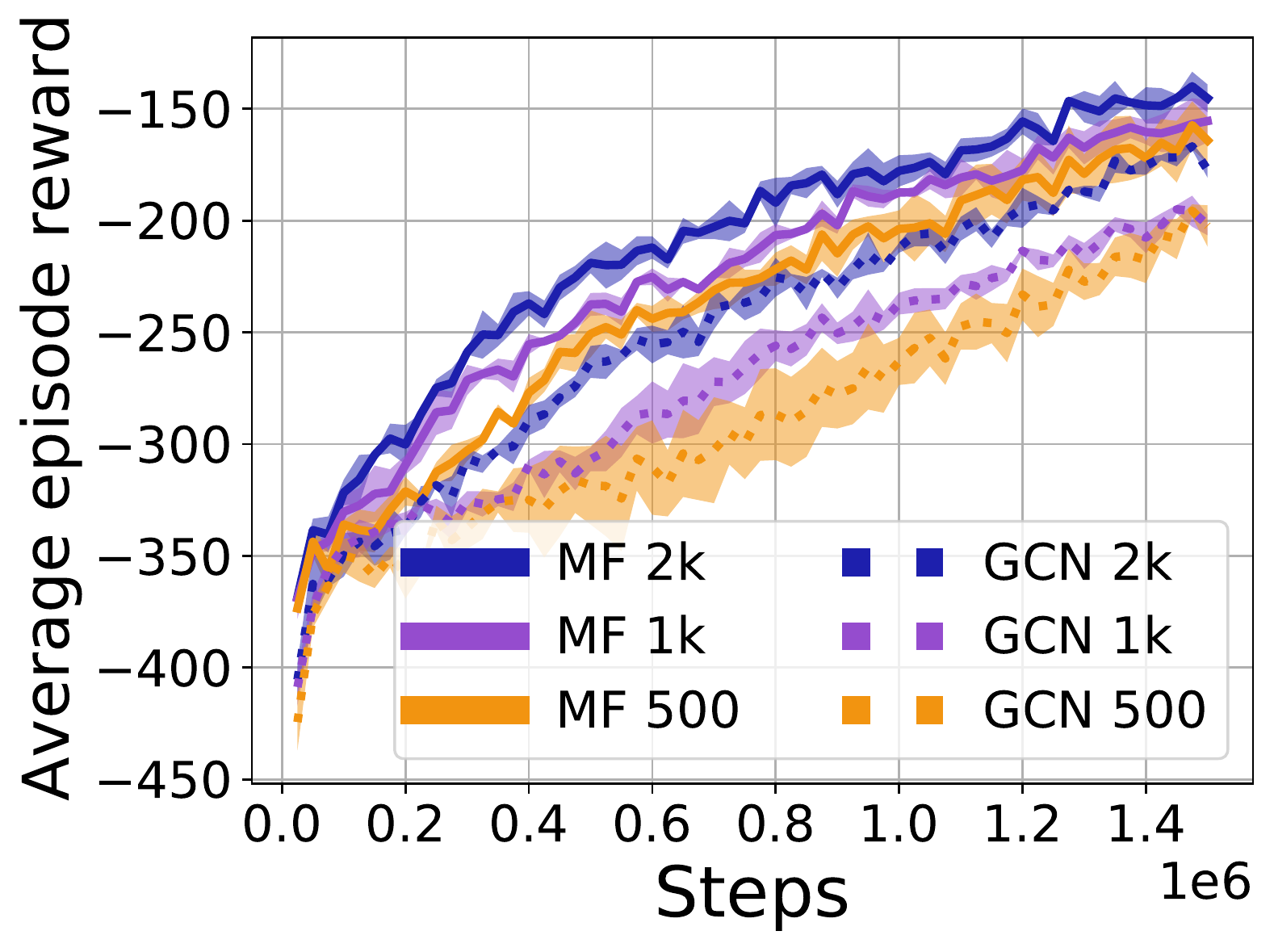}
    \vspace{-0.2in}
    \caption{Varying critic size}
    \label{fig:varying_critic}
  \end{subfigure}
  \hfill
\begin{subfigure}[t]{0.48\linewidth}
    \centering
    \includegraphics[width=1.0\linewidth]{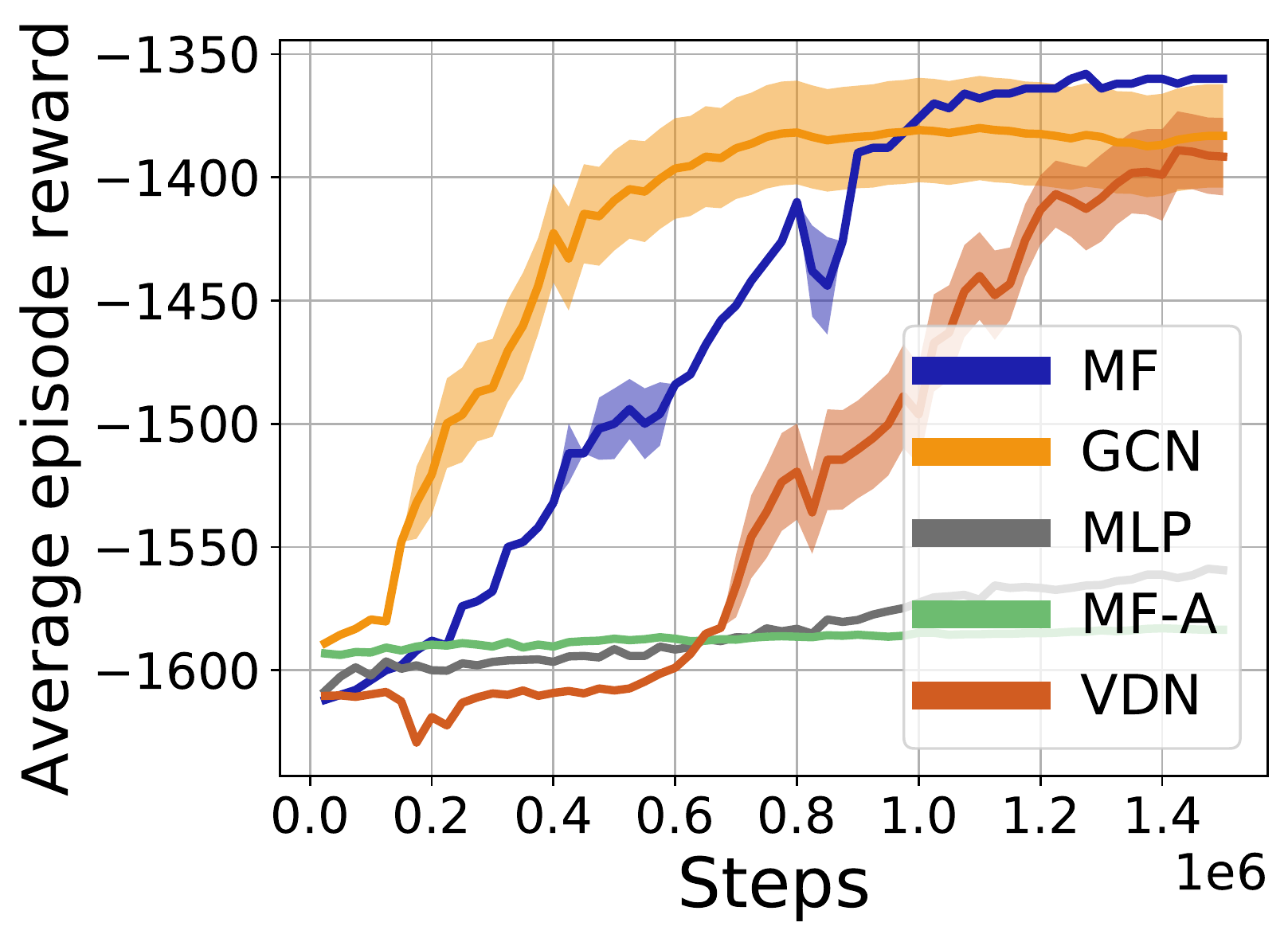}
    \vspace{-0.2in}
    \caption{$N=15$}
    \label{fig:coop_push_n15}
\end{subfigure}
\vspace{-0.05in}
\caption{ 
(a) MF outperforms GCN even with a fewer number of critic network parameters ($N=3$).
(b) Performance versus number of environment steps in the multi-agent \textit{cooperative push} task. 
 MF based on DeepSet invariant critic significantly outperforms other critic representations.
  }
  \vspace{-0.1in}
\end{figure}

\textbf{Computational Improvements}.
Theorem \ref{thrm:main} states that to obtain a small optimality gap in MF-PPO, one needs to compute the update on a large number of agents. We remark that with the dual embedding techniques developed in \citet{dai2017sbeed}, one can avoid computation on all the agents by sampling a small number of agents  to compute the update. This technique could be readily incorporated into MF-PPO to improve its computational efficiency.

\vspace{-0.1in}
\section{Conclusion}
\vspace{-0.05in}

We propose a principled approach to exploit agent homogeneity and permutation invariance through the mean-field approximation in MARL. 
Moreover, our results are also the first to provably show the global convergence of MARL algorithms with neural networks as function approximators. 
This is in sharp contrast to current practices, which are mostly heuristic methods without convergence guarantees. 


\bibliographystyle{ims}
\bibliography{references}

\newpage
\onecolumn 

\noindent\rule[0.5ex]{\linewidth}{4pt}
\begin{center}
{ \bf \Large Supplementary Material for A Permutation Invariant Approach to Deep Mean-Field Multi-Agent Reinforcement Learning}
\end{center}
\noindent\rule[0.5ex]{\linewidth}{1pt}
\appendix

\section{Algorithms for Policy Evaluation/Improvement}\label{append:alg}
\begin{algorithm}[H]
    \caption{Policy Evaluation via TD }
    \label{alg:evaluation}
    \begin{algorithmic}
    \STATE{\textbf{Require}: Mean-field MDP $(\overline{\cS}, \overline{\cA}, \PP,  r, \gamma)$, sample $\{(s_t, \mathbf{s}_t, \oa_t, s_t' ,\mathbf{s}_t', \oa_t')\}$, number of iterations $T$}
    \STATE{\textbf{Initialize:} $w_j(0) \sim \mathrm{Unif} \{-1, +1\}, ~~~ \theta_j(0) \sim \cN(0, \mathrm{I}_d/d), ~ \forall j \in [m_Q]$ }
      \STATE{Set step size $\eta \leftarrow T^{-1/2}$  }
    \FOR{$t=0, \ldots, T-1$}
    \STATE{ Set $(s, \mathbf{s} ,\oa, s', \mathbf{s}', \oa') \leftarrow (s_t, \mathbf{s}_{t} ,\oa_{t}, s_{t}', \mathbf{s}_{t}', \oa_{t}')$}
    \STATE{ $ \theta(t+1/2)  = \theta(t) - \eta \cbr{ F^Q_{\theta(t)}(s, \mathbf{s}, \oa) - (1-\gamma) r(s, \mathrm{d}_\cS, \oa) - \gamma F^Q_{\theta(t)}(s, \mathbf{s}', \oa')} \nabla_\theta F^Q_{\theta(t)}(s, \mathbf{s},\oa) $}
\STATE{ $\theta(t+1) = \Pi_{\mathbb{B}(R_\theta, \theta_0, )} ( \theta(t+1/2)  ) $}
    \ENDFOR
    \STATE{Take ergodic average: $\overline{\theta}_T \leftarrow \frac{1}{T} \sum_{t=0}^{T-1} \theta(t)$}
    \STATE{\textbf{Output}: $F^Q_{\overline{\theta}_T}$}
    \end{algorithmic}
\end{algorithm}

\begin{algorithm}[H]
    \caption{Policy Improvement via SGD}
    \label{alg:improvement}
    \begin{algorithmic}
    \STATE{\textbf{Require}: Mean-field MDP $(\overline{\cS}, \overline{\cA}, \PP,  r, \gamma)$, sample $\{(s_t, \mathbf{s}_t, \oa_t\}$, number of iterations $T$}
        \STATE{\textbf{Initialize:} $v_j(0) \sim \mathrm{Unif} \{-1, +1\}, ~~~ \alpha_j(0) \sim \cN(0, \mathrm{I}_d/d), ~\forall j \in [m_A]$ }
      \STATE{Set step size $\eta \leftarrow T^{-1/2}$  }
    \FOR{$t=0, \ldots, T-1$}
    \STATE{ Set $(s, \mathbf{s} ,\oa) \leftarrow (s_t, \mathbf{s}_{t} ,\oa_{t})$}
    \STATE{ $\alpha(t+1/2) = \alpha(t) - \eta \cbr{F^A_{\alpha(t)}(s, \mathbf{s}, \oa) - \tau_{k+1} ( \upsilon_k^{-1} F^{Q}_{\theta_k}(s, \mathbf{s}, \oa) + \tau_k^{-1} F^A_{\alpha_k}(s, \mathbf{s}, \oa))} \nabla_\alpha F^A_{ \alpha(t)}(s, \mathbf{s}, \oa)$}
\STATE{  $ \alpha(t+1) = \Pi_{\mathbb{B}(R_\alpha, \alpha_0)}  \rbr{ \alpha(t+1/2) }$}
    \ENDFOR
    \STATE{Take ergodic average: $\overline{\alpha}_T \leftarrow \frac{1}{T} \sum_{t=0}^{T-1} \alpha(t)$}
    \STATE{\textbf{Output}: $F^A_{\overline{\alpha}_T}$}
    \end{algorithmic}
\end{algorithm}

\section{Experimental details}
\label{app:experiment}

\noindent  $\diamondsuit$ \textbf{Environment.}
We used the open-source code for the multi-agent particle environments provided by \citet{liu2019pic}, which itself is based on the original code by \citet{lowe2017multi}, without any modification.
Please refer to \citet[Appendix B]{liu2019pic} for complete details.

\noindent  $\diamondsuit$ \textbf{Computing infrastructure and runtime.}
Experiments were run on Intel Xeon 6154 CPUs, using one core for each independent policy training session.
Average training time was approximately 4 hours for $N=15$ and $N=30$ with 1.5e6 steps, and 12 hours for $N=200$ with 7.25e5 steps.

\subsection{Implementation}

\noindent $\diamondsuit$ \textbf{GCN\footnote{We label this architecture as ``GCN'' to distinguish it from alternative ways to instantiate a permutation invariant critic.}.} 
For experiments based on MADDPG, we re-ran the open-source code provided by \citet{liu2019pic} without modification.
For experiments based on PPO, we used the GCN network in \citet{kipf2016semi} as the critic in PPO.
Performance of GCN reported in this work are the results of our experimental runs.

\noindent $\diamondsuit$ \textbf{MF.}
In the PPO-based implementation, the joint policy is parameterized by a multi-layer perceptron (MLP) with two hidden layers of size 128 and ReLU activation, which takes in the concatenation of all agents' observations and outputs the mean and diagonal covariance of a joint Gaussian policy.
In the MADDPG-based implementation, each decentralized actor (i.e., policy) network is an MLP with two hidden layers of size 187 and ReLU activation, which takes in each agent's individual observation and outputs the agent's real-valued action vector.
The centralized critic has the form $V(s) = f_3(f_2(\sum_{i=1}^N f_1(s_i)))$, where $f_1$ and $f_2$ are hidden layers of size $h=190/205/187/187/187$ (for $N=3/6/15/30/200$ agents, respectively, to have approximately equal number of trainable critic parameters as the GCN critic used in \citet{liu2019pic}) with ReLU activation, and $f_3$ is a linear layer with output size 1.

\noindent  $\diamondsuit$ \textbf{MLP.}
The centralized critic is an MLP with two hidden layers of size 138,187,75,42,7 for the case of $N=3,6,15,30,200$ agents, respectively, such that total number of trainable parameters is approximately equal to that in GCN \citep{liu2019pic}.
The joint policy has the same representation as the one in \textbf{MF} for both PPO-based and MADDPG-based implementations.

\noindent  $\diamondsuit$ \textbf{Hyperparameters.}
We used the Adam optimizer for all algorithms.
For experiments based on PPO, the common hyperparameters across all algorithms are: discount factor $\gamma=0.99$, KL divergence target $d_{\text{targ}}=0.01$, KL divergence initial coefficient $\beta=3$, and GAE $\lambda=0.95$ (see \citet{schulman2017proximal}), policy entropy loss coefficient 0.01, max gradient norm $0.5$, and actor learning rate $10^{-4}$.
Each PPO training step uses 5 epochs with minibatch size 2.
We did a sweep of critic learning rate for each critic architecture, choosing $10^{-3}$ for MF, $10^{-4}$ for GCN and MLP, and $10^{-2}$ for MF-A.

For experiments based on MADDPG, the common hyperparameters are:
actor learning rate $0.01$, critic learning rate $0.01$, batch size 1024, discount factor $\gamma=0.95$, replay buffer size $10^5$, 8 actor and critic updates per 100 environment steps, and target network update rate $\tau = 0.01$.


\section{Proofs in Section \ref{sec:background}}

\begin{proof}[Proof of Proposition \ref{prop:invariant_policy}]
We first prove the first part of claim. We begin by showing that the optimal Q-function $Q^*$ is permutation invariant, i.e., $Q^*(\sbb, \ab) = Q^*(\kappa(\sbb), \kappa(\ab))$ for all permutation mapping $\kappa$.
Note that $Q^*$ is the unique solution to Bellman optimality condition over the space of $Q$-function
\begin{align}
Q(\sbb, \ab) &= \EE_{\sbb' \sim \PP(\cdot| \sbb, \ab)} \cbr{r(\sbb, \ab) + \max_{\ab'} Q(\sbb', \ab') } \nonumber \\
&= r(\sbb, \ab) + \sum_{\sbb'} \PP(\sbb'|\sbb, \ab) \max_{\ab'} Q(\sbb', \ab'). \label{eq:bellmanq}
\end{align}
Or equivalently, from the permutation invariance of $r$ and $\PP$ defined in \eqref{eq:system}
\begin{align*}
Q(\kappa(\sbb), \kappa(\ab)) & = r(\kappa(\sbb), \kappa(\ab) ) + \sum_{\sbb'} \PP(\kappa(\sbb')|\kappa(\sbb), \kappa(\ab)) \max_{\ab'} Q(\kappa(\sbb'), \kappa(\ab')) \\
& =  r(\sbb, \ab) + \sum_{\sbb'} \PP(\sbb'|\sbb, \ab) \max_{\ab'} Q(\kappa(\sbb'), \kappa(\ab')).
\end{align*}
That is, $Q(\kappa(\sbb), \kappa(\ab))$ is also a solution to the Bellman optimality condition \eqref{eq:bellmanq}. From the uniqueness of the solution, we have 
$Q^*(\sbb, \ab) = Q^*(\kappa(\sbb), \kappa(\ab))$.  Hence the optimal policy $\nu^*(\sbb) = \argmax_{\ab} Q^*(\sbb, \ab)$ is permutation invariant.

For second part of the proposition, we show the detailed proof for permutation invariance of value function $V$.
For any permutation invariance policy $\nu$, let $R^\nu(\sbb) = \EE_{\ab \sim \nu} R(\sbb, \ab)$, and $\PP^\nu(\sbb'|\sbb) = \EE_{\ab \sim \nu} \PP(\sbb'|\sbb, \ab)$ denote the key elements of the induced Markov reward process. 
One can clearly see that from permutation invariance defined in \eqref{eq:system}, we have that both $R^\nu$ and $\PP^\nu(\sbb'|\sbb)$ is permutation invariant. 
We define the $k$-step value function $V^\nu_k$ as the expected reward from time $0$ to time $k$, then:
\begin{align*}
V^\nu_1 (\sbb) &= R^\nu (\sbb) ,\\
V^\nu_k (\sbb) & = R^\nu(\sbb) + \gamma \EE_{s' \sim \PP^\nu(\cdot|s)} V^\nu_{k-1} (\sbb'), ~~ \forall k > 1.
\end{align*}
We can see from above recursion that $V^\nu_k$ is permutation invariant for all $k\geq 1$. From $V^\nu = \lim_{k \to \infty} V^\nu_k$, we can conclude that $V^\nu$ is also permutation invariant. 

For permutation invariance of Q-function, recall that $Q^\nu(\sbb, \ab)$ is the unique solution of the following Bellman evaluation equation
\begin{align}
Q(\sbb, \ab) &= \EE_{\sbb' \sim \PP(\cdot| \sbb, \ab)} \cbr{r(\sbb, \ab) + \EE_{\ab' \sim \nu(\cdot| \sbb') } Q(\sbb', \ab') } \nonumber \\
&= r(\sbb, \ab) + \sum_{\sbb', \ab'} \PP(\sbb'|\sbb, \ab) Q(\sbb', \ab') \nu(\sbb' \ab') \label{eq:qeval}.
\end{align}
From the permutation invariance of $r$, $\PP$ and $\nu$, we have $Q^\nu$ is also a solution to \eqref{eq:qeval}, as
\begin{align*}
Q(\kappa(\sbb), \kappa(\ab)) = r(\sbb, \ab) + \sum_{\sbb', \ab'} \PP(\sbb'|\sbb, \ab) Q(\kappa(\sbb'), \kappa(\ab')) \nu(\sbb', \ab').
\end{align*}
Since $Q^\nu$ is the unique solution to \eqref{eq:qeval}, we have $Q^\nu(\sbb, \ab) = Q^\nu(\kappa(\sbb), \kappa(\ab))$.

\end{proof}

\begin{proof}[Proof of Proposition \ref{prop:val_func}]
The proof for showing $V^\nu(s, \sbb_r) = V^\nu(s, \kappa(\sbb_r))$ follows the exact same ingredients as in the proof for Proposition \ref{prop:invariant_policy}. 
Following Theorem 11 in \citet{bloem2019probabilistic}, there exists a function $h_\nu(s, \eta, \sum_{s' \in \sbb_r} \frac{1}{|\sbb_r|} \delta_{s'})$ such that $Y \overset{a.s.}{=} h_\nu$, where $\eta \sim \mathrm{Unif}[0,1]$. Then Proposition \ref{prop:val_func} follows from letting $g_\nu(s,  \sum_{s' \in \sbb_r} \frac{1}{|\sbb_r|} \delta_{s'}) = \EE_\eta h_\nu(\eta,  \sum_{s' \in \sbb_r} \frac{1}{|\sbb_r|} \delta_{s'}) $.

\end{proof}

\section{Proofs in Section \ref{sec:algo}}
\begin{proof}[Proof of Proposition \ref{prop:search_space}]
For any permutation invariant actor/critic function  $F(s, \sbb, \oa) = F(s, \kappa(\sbb), \oa)$,
we consider the size of  tabular representation for such functions and denote it as the size for the search space.
We say $\sbb$ and $\sbb'$ are \textit{permutation equivalent} if $\sbb = \kappa(\sbb')$ for some permutation mapping $\kappa$.
We can observe that (i) the permutation equivalent is a valid equivalent relation defined over the space of all possible $\sbb$; (ii) $\sbb$ and $\sbb'$ are permutation equivalent if and only if for every value $v$ in $\sbb$, $v$ occurs the same number of times in $\sbb$ and $\sbb'$. 
One can verify that for $\sbb$ containing $N$ elements, with each element $s'$ taking values in $\cS$, then the number of equivalence classes over the space of all possible $\sbb$,  induced by \textit{permutation equivalent} relation is  $\sum_{k=1}^{\min\{|\cS|, N\}} \binom{N-1}{k-1} \binom{|\cS|}{k}$. The claim of Proposition \ref{prop:search_space} follows immediately. 
\end{proof}

\begin{proof}[Proof of Proposition \ref{update:poly_regression}]
  Note that the maximization problem can be solved separately for each $s, \mathrm{d}_\cS$.
  Hence we can equivalently solve
\begin{align*}
\overline{\pi}(\cdot|s, \mathrm{d}_\cS) \max_{\pi(\cdot|s, \mathrm{d}_\cS) \in \RR^\cA} \sbr{ \inner{F^{Q}_{\theta_k}(s, \mathrm{d}_\cS, \cdot)}{\pi(\cdot | s, \mathrm{d}_\cS)} - \upsilon_k \mathrm{KL} \rbr{ \pi_\alpha(\cdot|s, \mathrm{d}_\cS) \| \pi(\cdot|s, \mathrm{d}_\cS)} },
\end{align*}
which is an optimization problem over finite-dimensional parameter $\pi(\cdot|s, \mathrm{d}_\cS) \in \RR^{\overline{\cA}}$.
Setting the derivative w.r.t. $\pi(\oa|s, \mathrm{d}_\cS)$ for each $\oa\in \overline{\cA}$ equal to zero, we have
\begin{align*}
F^{Q}_{\theta_k}(s, \mathrm{d}_\cS, \oa)  - \upsilon_k (\log\rbr{\frac{\underline{\pi}}{\pi_k(\oa|s, \mathrm{d}_\cS)}} - 1) = 0,
\end{align*}
or equivalently, $\overline{\pi}(a|s, \mathrm{d}_\cS)  \propto \pi_k(a|s, \mathrm{d}_\cS) \exp   \{\upsilon_k^{-1}F^{Q}_{\theta_k}(s, \mathrm{d}_\cS, \oa)   \}$. Plug in the definition of $\pi_{k} \propto \exp \rbr{\tau_k^{-1} F^{A}_{\alpha_k}}$ completes the proof.
\end{proof}

\section{Proofs in Section \ref{sec:theory}}

We provide a unified analysis of the policy evaluation and policy improvement step.
The policy evaluation and improvement steps can be described as minimizing the following mean-squared loss:
\begin{align}\label{opt:generic}
\beta(t+1) = \argmin_{\beta \in \mathbb{B}(R_\beta, \beta_0) } \EE_{(s, \mathrm{d}_\cS, a ) \sim \rho} \sbr{( F_{\beta}(s, \mathrm{d}_\cS, \oa) - \zeta_{F_\beta}(s, \mathrm{d}_\cS, \oa) )^2 } .
\end{align}
where the operator $\zeta$ maps function $F_\beta$ to 
\begin{align}\label{def:operator}
\zeta_{F_\beta}(s, \mathrm{d}_\cS, \oa) = \EE_{s', \mathrm{d}_\cS' \sim \PP(\cdot|s, \mathrm{d}_\cS, \oa),\oa' \sim \pi(\cdot|s', \mathrm{d}_\cS')} \sbr{\tau \xi(s, \mathrm{d}_\cS, \oa) + \mu F_{\beta}(s', \mathrm{d}_\cS', \oa')}.
\end{align}
Formulation \eqref{opt:generic} includes policy improvement and policy evaluation as special cases:

\noindent \textbf{Policy Improvement}. This corresponds to $\rho = \tilde{\sigma}_k$,  
$ \xi(s, \mathrm{d}_\cS, \oa) =  \tau_{k+1}( \upsilon_k^{-1} F^{Q}_{\beta_k}(s, \mathrm{d}_\cS, \oa) + \tau_k^{-1} F^A_{\alpha_k}(s, \mathrm{d}_\cS, \oa))$, $\tau = 1, \mu = 0$,

\noindent \textbf{Policy Evaluation}. This corresponds to $\rho = \sigma_k$, $\tau = 1-\gamma$, $\mu = \gamma$, $\xi(s, \mathrm{d}_\cS, \oa) = r(s, \mathrm{d}_\cS, \oa)$.
Note that $\zeta_{F_\theta}(s, \mathrm{d}_\cS, \oa) $ equals to the Bellman operator: $\zeta_{F_\theta}(s, \mathrm{d}_\cS, \oa) = \cT^\pi F_\theta(s, \mathrm{d}_\cS, \oa)$.

To solve  problem \eqref{opt:generic}, we consider the following generic update rule:
\begin{align}
 \beta(t+1/2) & = \rbr{F_{\beta(t)}(s, \mathrm{d}_\cS, \oa) - \tau \xi(s, \mathrm{d}_\cS, \oa) - \mu F_{\beta(t)}(s', \mathrm{d}_\cS', \oa') } \nabla_{\beta} F_{\beta(t)}(s, \mathrm{d}_\cS, \oa) ,\label{generic:grad} \\ 
 \beta(t+1) & = \Pi_{\mathbb{B}^0(R_\beta)} (\beta(t+1/2)) \label{generic:proj}.
\end{align}
where $(s, \mathrm{d}_\cS, \oa) \sim \rho, (s', \mathrm{d}_\cS') \sim \PP(\cdot| s, \mathrm{d}_\cS, \oa),\oa' \sim \pi(\cdot|s', \mathrm{d}_\cS')$.

Instead of knowing $s, \mathrm{d}_\cS$ exactly, we only have access to $s, \mathrm{d}_\cS$ via $N$ independent samples represented as the states of $N$ agents. Hence, we perform the following update:
\begin{align}
 \beta(t+1/2) & = \rbr{F_{\beta(t)}(s, \mathbf{s}, \oa) - \tau \xi(s, \mathrm{d}_\cS, \oa) - \mu F_{\beta(t)}(s', \mathbf{s'}, \oa') } \nabla_{\beta} F_{\beta(t)}(s, \mathbf{s}, \oa) ,\label{agent:grad} \\ 
 \beta(t+1) & = \Pi_{\mathbb{B}^0(R_\beta)} (\beta(t+1/2)). \label{agent:proj}
\end{align}
where $(s, \mathrm{d}_\cS, \oa) \sim \rho, (s', \mathrm{d}_\cS' ) \sim \PP(\cdot| s, \mathrm{d}_\cS, \oa), \oa' \sim \pi(\cdot|s', \mathrm{d}_\cS')$; and $\mathbf{s} \iid \mathrm{d}_\cS, \mathbf{s}' \iid \mathrm{d}_\cS'$.

\noindent \textbf{Policy Improvement Updates}. For $\rho = \tilde{\sigma}_k$, $\tau = 1, \mu = 0, R_\beta= R_\alpha$ and $\xi(s, \mathrm{d}_\cS, \oa) =  \tau_{k+1}( \beta_k^{-1} F^{Q}_{\theta_k} (s, \mathrm{d}_\cS, \oa) + \tau_k^{-1} F^A_{\alpha_k}(s, \mathrm{d}_\cS, \oa))$, we recover the policy improvement update in Algorithm \ref{alg:evaluation}.

\noindent \textbf{Policy Evaluation Updates}. For $\rho = \sigma_k$, $\tau = (1-\gamma), \mu = \gamma, \xi(s, \mathrm{d}_\cS, \oa) = r(s, \mathrm{d}_\cS, \oa), R_\beta = R_\theta$, we recover the policy evaluation update in Algorithm  \ref{alg:improvement}.

We define a few more notations before we proceed
\begin{align*}
\text{(residual):} & ~~~\delta_{\beta(t)}(s, \mathrm{d}_\cS, \oa, s', \mathrm{d}_\cS', \oa') = F_{\beta(t)}(s, \mathrm{d}_\cS, \oa) - \tau \xi(s, \mathrm{d}_\cS, \oa) - \mu F_{\beta(t)}(s', \mathrm{d}_\cS', \oa'), \\
\text{(stochastic semi-gradient):} & ~~~ g_{\beta(t)}(s, \mathrm{d}_\cS, \oa, s', \mathrm{d}_\cS', \oa') = \delta(s, \mathrm{d}_\cS, \oa, s', \mathrm{d}_\cS', \oa') \nabla_\beta  F_{\beta(t)}(s, \mathrm{d}_\cS, \oa), \\
\text{(population semi-gradient):} & ~~~ \overline{g}_{\beta(t)} = \EE_{s, \mathrm{d}_\cS, \oa, s', \mathrm{d}_\cS',\oa'} g(s, \mathrm{d}_\cS, \oa, s', \mathrm{d}_\cS', \oa'),
\end{align*}
where $(s, \mathrm{d}_\cS, \oa) \sim \rho, ( s', \mathrm{d}_\cS') \sim \PP(\cdot| s, \mathrm{d}_\cS, \oa),\oa' \sim \pi_k(\cdot | s', \mathrm{d}_\cS')$.

We also define the associated finite sample approximation:
\begin{align*}
\hat{\delta}_{\beta(t)}(s, \as, \oa, s',\as', \oa') & = F_{\beta(t)}(s, \mathbf{s}, \oa) - \tau \xi(s, \mathrm{d}_\cS, \oa) - \mu F_{\beta(t)}(s, \as', \oa'), \\ 
\hat{g}_{\beta(t)}(s, \as, a,s',  \as', \oa') & = \hat{\delta}(s, \as, \oa, s',  \as', \oa')  \nabla_\beta F_{\beta(t)}(s, \mathbf{s}, \oa),
\end{align*}
where $\as \stackrel{\mathrm{i.i.d.}}{\sim} \mathrm{d}_\cS,  \as' \stackrel{\mathrm{i.i.d.}}{\sim}  \mathrm{d}_\cS'$.
Note that given  $\hat{g}$ is not an unbiased estimator of $g$, as $\EE_{\as, \as'} \hat{g}(s, \as, \oa, s', \as', \oa') \neq g(s, \mathrm{d}_\cS, \oa, s', \mathrm{d}_\cS', \oa')$.

\subsection{Linearization at Initialization}
For a two layer ReLU network defined as $f_{\beta, u} (s, x, a ) = \frac{1}{\sqrt{m} }\sum_{j=1}^m u_j \sigma(\beta_j^\top(s, x, a))$, where $\sigma(x) = \max\{0,x\}$, we define 
its linearization at the initialization
\begin{align*}
f^0_{\beta} (s, x, a) =  \frac{1}{\sqrt{m}} \sum_{j=1}^m u_j  \indic \{\beta_j(0)^\top (s,x, a)>0\} \beta_j^\top(s,x, a).
\end{align*}
Note that $f^0_{\beta} (s, x, \oa)  = \inner{\nabla_{\beta} f_{\beta(0)} (s, x, a )}{\beta}$, which equivalently means $f^0_{\beta}$ is the local linearization of $f_{\beta}$ at its initialization $\beta(0)$ (note we do not train the second layer $u$).

We define
\begin{align*}
F^0_{\beta} (s, \as, \oa)  = \frac{1}{N} \sum_{i=1}^N f^0_{\beta} (s, s_i, \oa), ~~~F^0_{\beta} (s, \mathrm{d}_\cS, \oa)  = \EE_{x \sim \mathrm{d}_\cS} f^0_{\beta} (s, x, \oa),
\end{align*}
and similarly define $\delta^0, g^0 , \overline{g}^0, \hat{\delta}^0$, and $\hat{g}^0 $,  with everything related to $F_{\beta} ,F_{\beta}$ replaced by $F^0_{\beta}, F^0_{\beta}$.

Our main objectives are to show that the objective in \eqref{opt:generic} can be replaced by replacing everything related to $F$ with its local linearization $F^0$, and the linearized stochastic semi-gradient $\hat{g}^0 $ remains close to the population semi-gradient  $\overline{g}$ when the number of hidden units $m$ and number of agents are large enough.
By doing so, we can conclude that learning with linearized stochastic semi-gradient $\hat{g} $ is sufficient to solve \eqref{opt:generic} to high accuracy, when the number of hidden units $m$ and number of agents are large enough.
For the ease of exposition, we also make the following assumptions,

\begin{assumption}\label{assump:bounded_space}
We assume $\norm{(s, x, \oa)}_2 \leq 1$ for all $(s, x, \oa) \in \cS \times \cS \times \cA$.
\end{assumption}

\begin{assumption}\label{assump:reward}
The function $\xi(s, \mathrm{d}_\cS, \oa)$ satisfies:  for each $(s, \mathrm{d}_\cS) \in \overline{\cS} $ and any action $a$, we have
$(\xi(s, \mathrm{d}_\cS, \oa))^2 \leq \EE_{x \sim \mathrm{d}_\cS} \tau_1  (f_{\beta(0)} (s,x, \oa))^2 + \tau_2 R_\beta^2 + \tau_3$.
\end{assumption}

\textbf{Remark}: We will  verify that with   $(\tau_1, \tau_2, \tau_3) = (4,4,0)$, Assumption \ref{assump:reward} holds for the policy evaluation step. With $(\tau_1, \tau_2, \tau_3) = (0,0, \overline{r}^2)$,  Assumption \ref{assump:reward}  holds for the policy improvement step.

We first bound the difference between $F^0_{\beta}(s, \mathrm{d}_\cS, \oa)$ and  $F_{\beta} (s, \mathrm{d}_\cS, \oa)$.

\begin{lemma}\label{lemma:func_diff}
With Assumption \ref{assump:mild} and \ref{assump:bounded_space}, we have
\begin{align*}
\EE_{\mathrm{init}, \rho} |F^0_{\beta}(s, \mathrm{d}_\cS, \oa) - F_{\beta} (s, \mathrm{d}_\cS, \oa)|^2 \leq  \cO (R_\beta^3 m^{-1/2}).
\end{align*}
\end{lemma}

\begin{proof}
Be definition of $F^0_{\beta}(s, \mathrm{d}_\cS, \oa)$ and $F_{\beta} (s, \mathrm{d}_\cS, \oa)$ and Jensen's inequality, we have
\begin{align*}
\EE_{\mathrm{init}, s, \mathrm{d}_\cS, \oa } |F^0_{\beta}(s, \mathrm{d}_\cS, \oa) - F_{\beta} (s, \mathrm{d}_\cS, \oa)| & \leq  \EE_{\mathrm{init}, s, \mathrm{d}_\cS, \oa, x} |f^0_{\beta}(s, x, \oa) - f_{\beta} (s, x, \oa)|.
\end{align*}
We can further bound the right hand side with
\begin{align*}& 
\EE_{\mathrm{init}, s, \mathrm{d}_\cS, \oa, x} |f^0_{\beta}(s, x, \oa) - f_{\beta} (s, x, \oa)|  \\
= & \EE_{\mathrm{init}, s, \mathrm{d}_\cS, \oa, x}  \frac{1}{\sqrt{m}} \sum_{j=1}^m | \indic \{\beta_j(0)^\top (s, x, \oa) > 0 \} - \indic \{\beta_j^\top (s, x, \oa) >0\} |\beta_j^\top (s, x, \oa)| \\
\leq &  \EE_{\mathrm{init}, s, \mathrm{d}_\cS, \oa, x}  \frac{1}{\sqrt{m}} \sum_{j=1}^m   \abs{\indic \{\beta_j(0)^\top (s, x, \oa) > 0 \} - \indic \{\beta_j^\top (s, x, \oa) >0\} }
 \norm{\beta_j - \beta_j(0)}_2,
\end{align*}
where the last inequality comes from that fact that $\indic \{\beta_j(0)^\top (s, x, \oa) > 0 \} \neq \indic \{\beta_j^\top (s, x, \oa) >0\}$ implies $\abs{\beta_j^\top (s, x, \oa)} \leq   \abs{(\beta_j - \beta_j(0))^\top (s, x, \oa)} \leq    \norm{\beta_j - \beta_j(0)}_2 \norm{(s, x, \oa)}_2 \leq  \norm{\beta_j - \beta_j(0)}_2$.
Applying Cauchy Schwartz inequality, we have
\begin{align*}
& \EE_{\mathrm{init}, s, \mathrm{d}_\cS, a } |F^0_{\beta}(s, \mathrm{d}_\cS, \oa) - F_{\beta} (s, \mathrm{d}_\cS, \oa)|^2  \\ \leq 
 & \EE_{\mathrm{init}, s, \mathrm{d}_\cS, \oa, x}  \frac{1}{m} \rbr{\sum_{j=1}^m   | \indic \{\beta_j(0)^\top (s, x, \oa) > 0 \} - \indic \{\beta_j^\top (s, x, \oa) >0\} }
\rbr{\sum_{j=1}^m \norm{\beta_j - \beta_j(0)}_2^2}  \\
\leq &  \EE_{\mathrm{init}, s, \mathrm{d}_\cS, \oa, x}  \frac{R_\beta^2}{m} \sum_{j=1}^m  \indic \{ \abs{\beta_j(0)^\top (s, x, \oa)} < \norm{\beta_j(0) - \beta_j}_2\}  \\
\leq & \EE_{\mathrm{init}} \frac{cR_\beta^2}{m} \sum_{j=1}^m \frac{ \norm{\beta_j(0) - \beta_j}_2}{\norm{\beta_j(0)}_2} \\
\leq & \EE_{\init}  \frac{cR_\beta^2}{m} \rbr{\sum_{j=1}^m \norm{\beta_j(0) - \beta_j}_2^2}^{1/2}  \rbr{\sum_{j=1}^m \norm{\beta_j(0)}_2^{-2}}^{1/2}  =  \cO \rbr{\frac{cR_\beta^3}{m^{1/2}}},
\end{align*}
where in the third inequality we use Assumption \ref{assump:mild}.

\end{proof}

Next, we bound the difference between  $\pG$ and  and $\pLG$.

\begin{lemma}\label{lemma:diff_pg}
With Assumption  \ref{assump:mild}, \ref{assump:bounded_space} and \ref{assump:reward}, we have
\begin{align*}
\EE_{\mathrm{init}}\norm{\pG_\beta - \pLG_\beta}_2^2 \leq  \cO\rbr{ \frac{R_\beta^3}{m^{1/2}}}  .
\end{align*}
\end{lemma}

\begin{proof}
With decomposition
\begin{align*}
 g_\beta  - g^0_\beta= &  \rbr{F(s, \mathrm{d}_\cS, \oa) - F(s, \mathrm{d}_\cS, \oa) - \mu (F(s', \mathrm{d}_\cS', \oa') - F(s', \mathrm{d}_\cS', \oa'))} \nabla F(s, \mathrm{d}_\cS, \oa) \\
 + & \rbr{F(s, \mathrm{d}_\cS, \oa) - \tau \xi(s, \mathrm{d}_\cS, \oa) - \mu F(s', \mathrm{d}_\cS', \oa') }\rbr{\nabla_\beta F(s, \mathrm{d}_\cS, \oa) - \nabla_\beta F(s, \mathrm{d}_\cS, \oa)},
\end{align*}
we apply basic inequality $\norm{a+b}_2^2 \leq 2(\norm{a}_2^2 + \norm{b}_2^2)$ and obtain
\begin{align}\label{ineq:func_1}
\norm{\pG_\beta - \pLG_\beta}_2^2 \leq & 2 \rbr{\EE_{s, \mathrm{d}_\cS, \oa, s', \mathrm{d}_\cS',\oa'} |F(s, \mathrm{d}_\cS, \oa) - F(s, \mathrm{d}_\cS, \oa) - \mu (F(s', \mathrm{d}_\cS', \oa') - F(s', \mathrm{d}_\cS', \oa'))| \norm{\nabla F(s, \mathrm{d}_\cS, \oa)}_2}^2 + \nonumber \\
& 2 \rbr{\EE_{s, \mathrm{d}_\cS, \oa, s', \mathrm{d}_\cS',\oa'} |F(s, \mathrm{d}_\cS, \oa) - \tau \xi(s, \mathrm{d}_\cS, \oa) - \mu F(s', \mathrm{d}_\cS', \oa') | \norm{\nabla_\beta F(s, \mathrm{d}_\cS, \oa) - \nabla_\beta F(s, \mathrm{d}_\cS, \oa)}_2 }^2.
\end{align}
We have
$\nabla_{\beta} F(s, \mathrm{d}_\cS, \oa) = \frac{1}{\sqrt{m}}  \rbr{\indic \{\beta_1(0)^\top (s, x, \oa)>0\} (s, x, \oa), \ldots,  \indic \{\beta_m(0)^\top (s, x, \oa)>0\}(s, x, \oa) }$, and with the assumption that $\norm{(s, x, \oa)}_2 \leq 1$, we have $\norm{\nabla_{\beta} F(s, \mathrm{d}_\cS, \oa)}_2 \leq 1$. In addition, we have
\begin{align*}
\EE_{\init, s, \mathrm{d}_\cS, \oa, s', \mathrm{d}_\cS',\oa'} |F(s, \mathrm{d}_\cS, \oa) - F(s, \mathrm{d}_\cS, \oa) - \mu (F(s', \mathrm{d}_\cS', \oa') - F(s', \mathrm{d}_\cS', \oa'))|^2   = \cO(R_\beta^3 m^{-1/2}),
\end{align*}
as implied by Lemma \ref{lemma:func_diff}.
Hence the first term in \eqref{ineq:func_1} is of order $\cO(R_\beta^3 m^{-1/2})$.

To bound the second term in \eqref{ineq:func_1}, 
by Cauchy-Schwartz inequality we have
\begin{align*}
& \rbr{ \EE_{s, \mathrm{d}_\cS, \oa, s', \mathrm{d}_\cS',\oa'} |F(s, \mathrm{d}_\cS, \oa) - \tau \xi(s, \mathrm{d}_\cS, \oa) - \mu F(s', \mathrm{d}_\cS', \oa') | \norm{\nabla_\beta F(s, \mathrm{d}_\cS, \oa) - \nabla_\beta F(s, \mathrm{d}_\cS, \oa)}_2 }^2 \\
 \leq & \EE_{s, \mathrm{d}_\cS, \oa, s', \mathrm{d}_\cS',\oa'} |F(s, \mathrm{d}_\cS, \oa) - \tau \xi(s, \mathrm{d}_\cS, \oa) - \mu F(s', \mathrm{d}_\cS', \oa') |^2  \EE_{s, \mathrm{d}_\cS, a} \norm{\nabla_\beta F(s, \mathrm{d}_\cS, \oa) - \nabla_\beta F(s, \mathrm{d}_\cS, \oa)}_2^2.
\end{align*}

From Jensen's inequality we have
\begin{align*}
|F(s, \mathrm{d}_\cS, \oa) - \tau \xi(s, \mathrm{d}_\cS, \oa) - \mu F(s', \mathrm{d}_\cS', \oa') |^2 & \leq \EE |f^0_{\beta} (s, x, \oa) - \tau \xi(s, \mathrm{d}_\cS, \oa) - \mu f^0_{\beta} (s', x', \oa') |^2 \\
& \leq 3 \EE \sbr{(f^0_{\beta} (s, x, \oa))^2 + (\tau \xi(s, \mathrm{d}_\cS, \oa))^2 +(\mu f^0_{\beta} (s', x', \oa'))^2 }.
\end{align*}
Note that $\nabla f^0_\beta = \nabla f^0_{\beta(0)}$ for all $\beta$, we have
\begin{align*}
f^0_{\beta}(s, x, \oa) \leq  f^0_{\beta(0)} (s, x, \oa)  + \norm{\nabla_{\beta} f^0_{\beta}}_2 \norm{\beta - \beta(0)}_2 \leq   f^0_{\beta(0)} (s, x, \oa)  + R_\beta.
\end{align*}
We have 
$3 \rbr{f^0_{\beta} (s, x, \oa)^2 +(\mu f^0_{\beta} (s', x', \oa'))^2 } \leq 6 \rbr{f_{\beta(0)} (s, x, \oa)^2 +(\mu f_{\beta(0)} (s', x', \oa'))^2 } + 12 R_\beta^2 $.
Note that from Assumption \ref{assump:reward}:  $\xi(s, \mathrm{d}_\cS, \oa)^2 \leq \tau_1  \EE_{s, a} \sbr{(f_{\beta(0)} (s, x, \oa))^2 + \tau_2 R_\beta^2 + \tau_3}$.
Hence
\begin{align}
& \EE_{s, \mathrm{d}_\cS, \oa, s', \mathrm{d}_\cS',\oa'} \abs{F(s, \mathrm{d}_\cS, \oa) - \tau \xi(s, \mathrm{d}_\cS, \oa) - \mu F(s', \mathrm{d}_\cS', \oa') }^2  \nom \\
 \leq & \EE_{s, \mathrm{d}_\cS, \oa, s', \mathrm{d}_\cS',\oa'}    \EE_{x, x'} \sbr{ 6(f_{\beta(0)} (s, x, \oa)^2 + 6\mu^2 f_{\beta(0)} (s', x', \oa')^2 + 12 R_\beta^2 + 3\tau^2 \xi(s, \mathrm{d}_\cS, \oa)^2} \nom \\
  = & \EE_{s, \mathrm{d}_\cS, x, \oa} \sbr{ 6 f_{\beta(0)} (s, x, \oa)^2 + 6\mu^2 f_{\beta(0)} (s, x, \oa)^2 + 12 R_\beta^2 + 3\tau^2 \xi(s, \mathrm{d}_\cS, \oa)^2} \nom \\
 \leq & \EE_{s, \mathrm{d}_\cS, x, \oa}  \sbr{ 6 f_{\beta(0)} (s, x, \oa)^2 + 6\mu^2 f_{\beta(0)} (s, x, \oa)^2 + 12 R_\beta^2 + 3\tau^2 \tau_1 f_{\beta(0)} (s, x, \oa))^2 + 3\tau^2 \tau_2 R_\beta^2 + 3\tau^2\tau_3} \label{ineq:first} .
\end{align}
On the other hand, we have
\begin{align}
& \EE_{s, \mathrm{d}_\cS, a} \norm{\nabla_\beta F(s, \mathrm{d}_\cS, \oa) - \nabla_\beta F(s, \mathrm{d}_\cS, \oa)}_2^2  \nom \\
\leq & \frac{1}{m}\EE_{s, \mathrm{d}_\cS, a, x}\norm{  \rbr{ \indic \{\beta_1(0)^\top (s, x, \oa)>0\} - \indic \{\beta_1^\top (s, x, \oa)>0\}, \ldots,  \indic \{\beta_m(0)^\top (s, x, \oa)>0\} - \indic \{\beta_m^\top (s, x, \oa)>0\}}}_2^2 \norm{(s, x, \oa)}_2^2  \nom \\ 
 \leq &  \EE_{s, \mathrm{d}_\cS, \oa, x }  \frac{1}{m} \sum_{j=1}^m (\indic \{\beta_j(0)^\top (s, x, \oa)>0\} - \indic \{\beta_j^\top (s, x, \oa)>0\} )^2 \nom \\
 \leq &   \EE_{s, \mathrm{d}_\cS, \oa, x}\frac{1}{m} \sum_{j=0}^m \indic \{ \beta_j(0)^\top (s, x, \oa) \leq \norm{\beta_j(0) - \beta_j}_2 \} \nom\\
 \leq & \frac{c}{m} \sum_{j=1}^m\frac{\norm{\beta_j(0) - \beta_j}_2 }{\norm{\beta_j(0)}_2} \nom \\
 \leq & \frac{c}{m} \rbr{ \Bigl\lVert \sum_{j=1}^m\beta_j (0) - \beta_j \Bigr\rVert_2^2}^{1/2} \rbr{ \sum_{j=1}^m \norm{\beta_j(0)}_2^{-2}}^{1/2}\nom \\
 \leq & \frac{cR_\beta }{m} \rbr{ \sum_{j=1}^m \norm{\beta_j(0)}_2^{-2}}^{1/2} \label{ineq:sec},
\end{align}
where in the fourth inequality we use Assumption \ref{assump:richness},  and  in the final  inequality we use $\rbr{\norm{ \sum_{j=1}^m\beta_j (0) - \beta_j}_2^2}^{1/2} \leq R_\beta$.

Combining \eqref{ineq:first} and \eqref{ineq:sec}, to bound the second term in  \eqref{ineq:func_1}, it remains to bound the following:
\begin{align*}
& \EE_{s, \mathrm{d}_\cS, \oa, x}  \cbr{f_{\beta(0)} (s, x, \oa)^2 \rbr{ \sum_{j=1}^m \norm{\beta_j(0)}_2^{-2}}^{1/2} \left(\frac{cR_\beta}{m} \right) }\\
 = &\frac{cR_\beta}{m^2} \EE_{s, \mathrm{d}_\cS, \oa, x}  \rbr{ \sum_{j=1}^m \sigma^2( \beta_j(0)^\top (s, x, \oa) )  +  \sum_{k \neq l}^m u_k u_l \sigma( \beta_k(0)^\top (s, x, \oa)) \sigma(\beta_l(0)^\top (s, x, \oa) )  }  \rbr{ \sum_{j=1}^m \norm{\beta_j(0)}_2^{-2}}^{1/2} \\
 \leq & \frac{cR_\beta}{m^2}\EE_{s, \mathrm{d}_\cS, \oa, x}   \rbr{\sum_{j=1}^m \norm{\beta_j(0)}_2^2 + \sum_{k \neq l}^m u_k u_l \sigma( \beta_k(0)^\top (s, x, \oa)) \sigma(\beta_l(0)^\top (s, x, \oa) )} \rbr{ \sum_{j=1}^m \norm{\beta_j(0)}_2^{-2}}^{1/2} ,
\end{align*}
Taking expectation with respect to initialization and noticing that $\EE_{\init} \cbr{u_k u_l }= 0$ for $k \neq l$, we have
\begin{align*}
&\frac{cR_\beta}{m^2} \EE_{\init, s, \mathrm{d}_\cS, \oa, x}  \rbr{\sum_{j=1}^m \norm{\beta_j(0)}_2^2 + \sum_{k \neq l}^m u_k u_l \sigma( \beta_k(0)^\top (s, x, \oa)) \sigma(\beta_l(0)^\top (s, x, \oa) )}   \rbr{ \sum_{j=1}^m \norm{\beta_j(0)}_2^{-2}}^{1/2} \\
= & \frac{c R_\beta}{m^2} \EE_{\init}  \rbr{\sum_{j=1}^m \norm{\beta_j(0)}_2^2 }   \rbr{ \sum_{j=1}^m \norm{\beta_j(0)}_2^{-2}}^{1/2}  \\
\leq & \frac{c R_\beta}{m^2}   \EE_{\init}^{1/2}  \rbr{\sum_{j=1}^m \norm{\beta_j(0)}_2^2 }^2   \EE_{\init}^{1/2}   \rbr{ \sum_{j=1}^m \norm{\beta_j(0)}_2^{-2}} \\
= & \cO \rbr{\frac{cR_\beta}{m^{1/2}}}.
\end{align*}
Then the the second term in \eqref{ineq:func_1} is at the order of $\cO \left( \max \{ \frac{R_\beta^3}{m^{1/2}},  \frac{R_\beta}{m^{1/2}} \} \right) = \cO \left( \frac{R_\beta^3}{m^{1/2}} \right)$.

\end{proof}
We then bound the the variance of stochastic semi-gradient $g_\beta(s, \mathrm{d}_\cS, \oa, s', \mathrm{d}_\cS', \oa')$.
\begin{lemma}\label{lemma:var_sg}
With Assumption  \ref{assump:reward}, there exists $\varsigma^2 = \cO(R_\beta^2)$, such that:
\begin{align*}
\EE_{\init, s, \mathrm{d}_\cS, \oa, s', \mathrm{d}_\cS',\oa'} \norm{g_\beta(s, \mathrm{d}_\cS, \oa, s', \mathrm{d}_\cS', \oa') - \overline{g}_\beta}_2^2 \leq \varsigma^2.
\end{align*}
\end{lemma}

\begin{proof}
We have
\begin{align*}
\EE_{s, \mathrm{d}_\cS, \oa, s', \mathrm{d}_\cS',\oa'} \norm{g_\beta(s, \mathrm{d}_\cS, \oa, s', \mathrm{d}_\cS', \oa') - \overline{g}_\beta}_2^2 & \leq  \EE_{s, \mathrm{d}_\cS, \oa, s', \mathrm{d}_\cS',\oa'} \norm{g_\beta(s, \mathrm{d}_\cS, \oa, s', \mathrm{d}_\cS', \oa') }_2^2 \\
& \leq \EE_{s, s, \mathrm{d}_\cS, s', \mathrm{d}_\cS',\oa'} \norm{ \delta(s, \mathrm{d}_\cS, \oa, s', \mathrm{d}_\cS', \oa') \nabla_\beta  F_{\beta(t)}(s, \mathrm{d}_\cS, \oa)}_2^2.
\end{align*}
With $\nabla_{\beta} F_{\beta(t)} (s, \mathrm{d}_\cS, \oa)  = \EE_{x} \nabla_{\beta} f_{\beta(t)} (s, x, \oa) =  \EE_{x} \frac{1}{\sqrt{m}} \rbr{ \indic\{\beta_j^\top(s, x, \oa) > 0\}, \ldots, \indic\{\beta_j^\top(s, x, \oa) > 0\} } (s, x, \oa)$, we have $\norm{\nabla_{\beta} F_{\beta(t)} (s, \mathrm{d}_\cS, \oa)  }_2 \leq 1$.
Then
\begin{align*}
 & \EE_{s, \mathrm{d}_\cS, \oa, s', \mathrm{d}_\cS',\oa'} \norm{g_\beta(s, \mathrm{d}_\cS, \oa, s', \mathrm{d}_\cS', \oa') - \overline{g}_\beta}_2^2  \\
 \leq & \EE_{s, s, \mathrm{d}_\cS, s', \mathrm{d}_\cS',\oa'} \sbr{ \delta^2(s, \mathrm{d}_\cS, \oa, s', \mathrm{d}_\cS', \oa')} \\
=& \EE_{s, \mathrm{d}_\cS, \oa, s', \mathrm{d}_\cS',\oa'} \rbr{F_{\beta}(s, \mathrm{d}_\cS, \oa) - \tau \xi(s, \mathrm{d}_\cS, \oa) - \mu F_{\beta}(s', \mathrm{d}_\cS', \oa')}^2 \\
 \leq & \EE_{s, \mathrm{d}_\cS, \oa, s', \mathrm{d}_\cS',\oa', x, x'} \rbr{f_{\beta}(s, x, \oa) - \tau \xi(s, \mathrm{d}_\cS, \oa) - \mu f_{\beta}(s', x', \oa')}^2 \\
 \leq  &3 \EE_{s, \mathrm{d}_\cS, \oa, x} \sbr{(1+\mu^2) \rbr{f_{\beta}(s, x, \oa) }^2 + \rbr{\tau \xi(s, \mathrm{d}_\cS, \oa)}^2} \\
\leq & 3 \EE_{s, \mathrm{d}_\cS, \oa, x} \sbr{(1+\mu^2) \rbr{f_{\beta(0)}(s, x, \oa) }^2 + \tau_1\rbr{\tau f_{\beta(0)}(s, x, \oa)}^2} + \cO(R_\beta^2) + \cO(\tau_3).
\end{align*}
where in the last inequality we use $\abs{f_{\beta(0)}(s, x, \oa) - f_{\beta}(s, x, \oa)} \leq \norm{\beta - \beta(0)}_2 \leq R_\beta$ and Assumption  \ref{assump:reward}.
Finally, note that 
\begin{align*}
\EE_{\init, s, \mathrm{d}_\cS, \oa, x} \sbr{f_{\beta(0)} (s, x, \oa)^2} & = \EE_{\init, s, \mathrm{d}_\cS, \oa, x}  \frac{1}{m} \rbr{ \sum_{j=1}^m \sigma^2( \beta_j(0)^\top (s, x, \oa) )  +  \sum_{k \neq l}^m u_k u_l \sigma( \beta_k(0)^\top (s, x, \oa)) \sigma(\beta_l(0)^\top (s, x, \oa) )  } \\
& =  \EE_{\init, s, \mathrm{d}_\cS, \oa, x}  \frac{1}{m} \rbr{ \sum_{j=1}^m \sigma^2( \beta_j(0)^\top (s, x, \oa) ) } \\
& \leq \EE_{\init} \frac{1}{m} \rbr{ \sum_{j=1}^m  \norm{\beta_j(0)}_2^2 } = 1.
\end{align*}
The claim follows immediately.
\end{proof}

We then bound the variance of $\hat{g}(s, \mathbf{s}, a, s', \mathbf{s'}, \oa')$ for fixed $\mathrm{d}_\cS, \mathrm{d}_\cS'$.
\begin{lemma}\label{lemma:var_ssg}
For any $s, \mathrm{d}_\cS, \oa, s', \mathrm{d}_\cS',\oa'$, let $\mathbf{s} \iid s, \mathrm{d}_\cS$ and $\mathbf{s}' \iid s', \mathrm{d}_\cS'$, and $|\mathbf{s}| = |\mathbf{s}'| = N$,  
we have
\begin{align*}
\EE_{\init,  \mathbf{s}, \mathbf{s}'} \norm{\hat{g}_\beta(s, \mathbf{s}, \oa, s', \mathbf{s'},\oa' ) - g_\beta(s, \mathrm{d}_\cS, \oa, s', \mathrm{d}_\cS', \oa')}_2^2 \leq \cO\rbr{\frac{R_\beta^2}{N}}.
\end{align*}
\end{lemma}
\begin{proof}
We have
\begin{align}\label{ineq:decomp}
& \hat{g}_\beta(s, \mathbf{s}, \oa, s', \mathbf{s'},\oa' ) - g_\beta(s, \mathrm{d}_\cS, \oa, s', \mathrm{d}_\cS', \oa')  \nonumber \\
= &  \hat{\delta}_\beta(s, \mathbf{s}, \oa, s', \mathbf{s'}, \oa')   \nabla_\beta \sum_{s \in \mathbf{s}} f_\beta(s, x, \oa) /N - \delta_\beta(s, \mathrm{d}_\cS, \oa, s', \mathrm{d}_\cS', \oa') \nabla_\beta \EE_{x \sim \mathrm{d}_\cS} f_\beta(s, x, \oa) \nonumber  \\
= &  \rbr{ \hat{\delta}_\beta(s, \mathbf{s}, \oa, s', \mathbf{s'}, \oa')  - \delta_\beta(s, \mathrm{d}_\cS, \oa, s', \mathrm{d}_\cS', \oa')  }  \nabla_\beta \sum_{s \in \mathbf{s}} f_\beta(s, x, \oa) /N \nonumber \\
 & + \delta_\beta(s, \mathrm{d}_\cS, \oa, s', \mathrm{d}_\cS', \oa')\rbr{    \sum_{s \in \mathbf{s}}  \nabla_\beta f_\beta(s, x, \oa) /N -  \nabla_\beta \EE_{x \sim \mathrm{d}_\cS} f_\beta(s, x, \oa) } .
\end{align}
For the first term,  note that $\norm{ \sum_{s \in \mathbf{s}}  \nabla_\beta  f_\beta(s, x, \oa) /N}_2 \leq 1$, we have
\begin{align*}
& \EE_{\mathbf{s} , \mathbf{s}'} \Bigl\lVert \rbr{ \hat{\delta}_\beta(s, \mathbf{s}, \oa, s', \mathbf{s'}, \oa')  - \delta_\beta(s, \mathrm{d}_\cS, \oa, s', \mathrm{d}_\cS', \oa')  }  \sum_{s \in \mathbf{s}}  \nabla_\beta f_\beta(s, x, \oa) /N \Bigr\rVert_2^2 \\
\leq & \EE_{\mathbf{s} , \mathbf{s}'}  \rbr{ \hat{\delta}_\beta(s, \mathbf{s}, \oa, s', \mathbf{s'}, \oa')  - \delta_\beta(s, \mathrm{d}_\cS, \oa, s', \mathrm{d}_\cS', \oa') } ^2\\
\leq & \EE_{x , x'} \rbr{f_\beta(s, x, \oa) - \tau \xi(s, \mathrm{d}_\cS, \oa) - \mu f_\beta(s', x', \oa')}^2/N \\
\leq & 3 \EE_{x , x'} \sbr{ \rbr{f_{\beta(0)}(s, x, \oa) }^2 + \mu^2 \rbr{f_{\beta(0)}(s', x', \oa') }^2+ \tau_1\rbr{\tau f_{\beta(0)}(s, x, \oa)}^2}/N + \cO(R_\beta^2/N) + \cO(\tau_3).\\
\end{align*}
where the last inequality we use the similar argument in Lemma \ref{lemma:var_sg}.
Now as we have shown in Lemma \ref{lemma:var_sg}, we have $\EE_{\init, s'} \sbr{f_{\beta(0)}(s, x, \oa)^2 } \leq 1$.
Hence the first term in  \eqref{ineq:decomp} is  at the order of $\cO(R_\beta^2/N)$.

To bound the second term in \eqref{ineq:decomp}, from Jensen's inequality we have
\begin{align*}
& \EE_\init \delta_\beta(s, \mathrm{d}_\cS, \oa, s', \mathrm{d}_\cS', \oa')^2 \\
\leq  &\EE_{\init, x, x'} \rbr{ f_{\beta}(s, x, \oa) - \tau \xi(s, \mathrm{d}_\cS, \oa) - \mu f_{\beta} (s',x', \oa')  }^2 \\
 \leq & 3 \EE_{\init, x, x'} \sbr{ \rbr{f_{\beta(0)}(s, x, \oa) }^2 + \mu^2 \rbr{f_{\beta(0)}(s', x', \oa') }^2+ \tau_1\rbr{\tau f_{\beta(0)}(s, x, \oa)}^2} + \cO(R_\beta^2) + \cO(\tau_3) \\
 = & \cO \rbr{R_\beta^2}.
\end{align*}
On the other hand, we have
\begin{align*}
\EE_{\mathbf{s}}  \Bigl\lVert   \sum_{s \in \mathbf{s}} \nabla_\beta  f_\beta(s, x, \oa) /N -  \nabla_\beta \EE_{x \sim \mathrm{d}_\cS} f_\beta(s, x, \oa) \Bigr\rVert_2^2 
\leq \EE_{x  \sim \mathrm{d}_\cS} \norm{ \nabla_\beta  f_\beta(s, x, \oa)}_2^2/N \leq 1/N.
\end{align*}

Therefore, taking square on both sides of \eqref{ineq:decomp} and using Cauchy Schwartz inequality, we have
\begin{align*}
&\EE_{\init,  \mathbf{s}, \mathbf{s}'} \norm{\hat{g}_\beta(s, \mathbf{s}, \oa, s', \mathbf{s'},\oa' ) - g_\beta(s, \mathrm{d}_\cS, \oa, s', \mathrm{d}_\cS', \oa)}_2^2
\\ \leq & 2 \EE_{\init,  \mathbf{s}, \mathbf{s}'}  \rbr{ \hat{\delta}_\beta(s, \mathbf{s}, \oa, s', \mathbf{s'}, \oa')  - \delta_\beta(s, \mathrm{d}_\cS, \oa, s', \mathrm{d}_\cS', \oa')  }^2 \Bigl\lVert  \sum_{s \in \mathbf{s}}   \nabla_\beta f_\beta(s, x, \oa) /N \Bigr\rVert_2^2 \\
& +  2 \EE_{\init,  \mathbf{s}, \mathbf{s}'} \cbr{ \delta^2 _\beta(s, \mathrm{d}_\cS, \oa, s', \mathrm{d}_\cS', \oa') }  \EE_{\init, x} \Bigl\lVert    \sum_{s \in \mathbf{s}} \nabla_\beta f_\beta(s, x, \oa) /N -  \nabla_\beta \EE_{x \sim \mathrm{d}_\cS} f_\beta(s, x, \oa) \Bigr\rVert_2^2  \\
 = &  \cO(R_\beta^2/N).
\end{align*}
\end{proof}

With Lemma \ref{lemma:var_sg} and Lemma \ref{lemma:var_ssg}, we can now bound the different between $\hat{g}_\beta$ and $\overline{g}_\beta$.

\begin{lemma}\label{lemma:mse_ssg}
With Assumption  \ref{assump:reward}, there exists $\varsigma^2 = \cO(R_\beta^2)$, such that
\begin{align*}
\EE_{\init, s, \mathrm{d}_\cS,\oa, s', \mathrm{d}_\cS' ,\oa' , \mathbf{s}, \mathbf{s}' } \norm{\hat{g}_\beta- \overline{g}_\beta}_2^2 \leq   \varsigma^2 + \cO(R_\beta^2/N).
\end{align*}
\end{lemma}
\begin{proof}
We have
\begin{align*}
 \EE \norm{\hat{g}_\beta - \overline{g}_\beta}_2^2  
\leq & 2 \EE \norm{\hat{g}_\beta(s, \mathbf{s}, a, s', \mathbf{s}' ,a') - g_\beta(s, \mathrm{d}_\cS, \oa, s', \mathrm{d}_\cS', \oa') }_2^2 \\ & + 2 \EE \norm{g_\beta(s, \mathrm{d}_\cS, \oa, s', \mathrm{d}_\cS', \oa') - \overline{g}_\beta}_2^2
\end{align*}
Now apply Lemma \ref{lemma:var_sg} and Lemma \ref{lemma:var_ssg}, the claim follow immediately. 
\end{proof}

We are now ready to show the convergence of using updates \eqref{generic:grad} and \eqref{generic:proj} to solve problem \eqref{opt:generic}.
We denote $F^0_{\beta_*}$ to be the function within the class $\cF_{R,m}$, which satisfies the stationary condition:
\begin{align}\label{cond:stationary}
\Pi_{\mathbb{B}(R_\beta, \beta_0)} (\beta_* - \overline{g}^0_{\beta_*}) = \beta_*, 
\end{align}
where $\Pi_X(\cdot)$ denotes the Euclidean projection onto set $X$.
Condition \eqref{cond:stationary} is equivalent to:
\begin{align*}
\inner{\overline{g}^0_{\beta_*}}{\beta -\beta_*} \geq 0.
\end{align*}
Note that $\overline{g}^0_{\beta_*} = \EE_{s, \mathrm{d}_\cS, \oa, s', \mathrm{d}_\cS',\oa'} \cbr{\delta_{\beta_*}^0(s, \mathrm{d}_\cS, \oa, s', \mathrm{d}_\cS', \oa') \nabla F^0_{\beta_*}(s, \mathrm{d}_\cS, \oa)}$, and $\inner{\nabla F^0_{\beta_*}(s, \mathrm{d}_\cS, \oa)}{\beta - \beta_*} = F^0_{\beta}(s, \mathrm{d}_\cS, \oa) - F^0_{\beta_*}(s, \mathrm{d}_\cS, \oa)$.
Hence we have
\begin{align}\label{eq:prep_1}
\EE_{s, \mathrm{d}_\cS, \oa, s', \mathrm{d}_\cS',\oa'} \cbr{  \delta_{\beta_*}^0(s, \mathrm{d}_\cS, \oa, s', \mathrm{d}_\cS', \oa') \rbr{F^0_{\beta,u}(s, \mathrm{d}_\cS, \oa) - F^0_{\beta_*}(s, \mathrm{d}_\cS, \oa)} }\geq 0.
\end{align}

With the definition of operator  $ \zeta_F$ in \eqref{def:operator},
we know that 
\begin{align}\label{eq:prep_2}
\EE_{s', \mathrm{d}_\cS' ,\oa' } \delta_{\beta_*}^0(s, \mathrm{d}_\cS, \oa, s', \mathrm{d}_\cS', \oa') = F^0_{\beta_*} (s, \mathrm{d}_\cS, \oa) - \zeta_{F^0_{\beta_*}}(s, \mathrm{d}_\cS, \oa).
\end{align}
From \eqref{eq:prep_1} and \eqref{eq:prep_2}, we  have
\begin{align*}
\inner{F^0_{\beta_*} - \zeta_{F^0_{\beta_* }}}{ F^0_{\beta}- F^0_{\beta_*}}_{\rho} \geq 0,
\end{align*}
which is equivalent to that $F^0_{\beta_*} = \Pi_{F_{R_\beta, m}}(\zeta_{F^0_{\beta_* }})$. That is, $F^0_{\beta_*}$ is the projection (with metric defined w.r.t $\inner{\cdot}{\cdot}_{\rho}$), after applying operator $ \zeta $ to itself.

\begin{theorem}\label{theorem:generic_opt}
Let $\{\beta(t)\}_{t=0}^{T-1}$ be generated by updates \eqref{agent:grad} and \eqref{agent:proj}.
Define $\overline{\beta}_T = \frac{1}{T} \sum_{t=0}^{T-1} \beta(t)$, we have
\begin{align*}
\EE_{\init, s, \mathrm{d}_\cS, \oa} \sbr{F^0_{\beta_*} (s, \mathrm{d}_\cS, \oa) - F_{\overline{\beta}_T}(s, \mathrm{d}_\cS, \oa)}^2  \leq \cO \rbr{\frac{R_\beta^2}{T^{1/2}} + \frac{R_\beta^{5/2}}{m^{1/4}} + \frac{R_\beta^2}{N^{1/2}}   + \frac{R_\beta^3}{m^{1/2}} }.
\end{align*}
\end{theorem}

\begin{proof}
Conditioned on the $t$-th iteration, we have
\begin{align}
& \EE_{s, \mathrm{d}_\cS, \oa, s', \mathrm{d}_\cS',\oa',\mathbf{s}, \mathbf{s}'} \sbr{\norm{\beta(t+1) - \beta_*}_2^2 \vert \beta(t)}  \nom\\
 =& \norm{\Pi_{\mathbb{B}^0_{R_\beta}}(\beta - \eta \hat{g}_{\beta(t)} ) - \Pi_{\mathbb{B}^0_{R_\beta}}(\beta_* - \eta\overline{g}^0_{\beta_*} )}_2^2  \nonumber \\
 \leq & \EE \norm{\beta(t) - \eta \hat{g}_{\beta(t)}  - \beta_* - \eta\overline{g}^0_{\beta_*} }_2^2 \nonumber \\
 = & \EE\rbr{\norm{\beta(t) - \beta_*}_2^2 - \eta \inner{\beta(t) - \beta_*}{ \hat{g}_{\beta(t)} - \overline{g}^0_{\beta_*}} + \eta^2 \norm{\hat{g}_{\beta(t)} - \overline{g}^0_{\beta_*}}_2^2 }\label{eq:convergence},
\end{align}
where the inequality comes from the non-expansive property of projection.

We first consider the second term in  \eqref{eq:convergence}, we have
\begin{align}
& \EE_{s, \mathrm{d}_\cS, \oa, s', \mathrm{d}_\cS',\oa',\mathbf{s}, \mathbf{s}'} \inner{\beta(t) - \beta_*}{ \hat{g}_{\beta(t)} - \overline{g}^0_{\beta_*}} \nom\\
 =&  \EE \sbr{\inner{\beta(t) - \beta_*}{ \hat{g}_{\beta(t)} - g_{\beta(t)}} +  \inner{\beta(t) - \beta_*}{ g_{\beta(t)} - \overline{g}_{\beta(t)}} 
+  \inner{\beta(t) - \beta_*}{ \overline{g}_{\beta(t)} - \overline{g}^0_{\beta_*}} }\nom \\
= &  \EE \inner{\beta(t) - \beta_*}{ \hat{g}_{\beta(t)} - g_{\beta(t)}} 
+ \inner{\beta(t) - \beta_*}{ \overline{g}_{\beta(t)} - \overline{g}^0_{\beta_*}} \nom \\
 \geq &  \EE \inner{\beta(t) - \beta_*}{ \overline{g}_{\beta(t)} - \overline{g}^0_{\beta_*}}  - R_\beta  \rbr{\EE \norm{\hat{g}_{\beta(t)} - g_{\beta(t)}}_2^2}^{1/2} \nom \\
 \geq  & \EE \inner{\beta(t) - \beta_*}{ \overline{g}_{\beta(t)} - \overline{g}^0_{\beta(t)}}
+ \inner{\beta(t) - \beta_*}{ \overline{g}^0_{\beta(t)} - \overline{g}^0_{\beta_*}} - R_\beta  \rbr{\EE \norm{\hat{g}_{\beta(t)} - g_{\beta(t)}}_2^2}^{1/2}, \label{ineq:merge}
\end{align}
where in the first equality we use definition $\EE g_{\beta(t)} = \overline{g}_{\beta(t)}$.
We further have 
\begin{align}\label{ineq:sep_1}
 \inner{\beta(t) - \beta_*}{ \overline{g}_{\beta(t)} - \overline{g}^0_{\beta(t)}} \geq 
-R_\beta \norm{\overline{g}_{\beta(t)} - \overline{g}^0_{\beta(t)}}_2.
\end{align}
Using again $\overline{g}^0_{\beta} = \EE \cbr{ \delta_{\beta_*}^0(s, \mathrm{d}_\cS, \oa, s', \mathrm{d}_\cS', \oa') \nabla F^0_{\beta}(s, \mathrm{d}_\cS, \oa)}$, $\nabla F^0_{\beta}(s, \mathrm{d}_\cS, \oa) = \nabla F^0_{\beta_*}(s, \mathrm{d}_\cS, \oa)$ for all $\beta$, and $\inner{\nabla F^0_{\beta}(s, \mathrm{d}_\cS, \oa)}{\beta - \beta_*} = F^0_{\beta}(s, \mathrm{d}_\cS, \oa) - F^0_{\beta_*}(s, \mathrm{d}_\cS, \oa)$, we have
\begin{align}
& \EE_{\init, s, \mathrm{d}_\cS, \oa, s', \mathrm{d}_\cS',\oa'} \inner{\beta(t) - \beta_*}{ \overline{g}^0_{\beta(t)} - \overline{g}^0_{\beta_*}} \nom \\
 = & \EE \rbr{\delta^0_{\beta(t)}(s, \mathrm{d}_\cS, \oa, s', \mathrm{d}_\cS', \oa') - \delta^0_{\beta_*}(s, \mathrm{d}_\cS, \oa, s', \mathrm{d}_\cS', \oa')} \rbr{F^0_{\beta(t)} (s, \mathrm{d}_\cS, \oa) - F^0_{\beta_*}(s, \mathrm{d}_\cS, \oa)} \nom \\
 =  & \EE \rbr{F^0_{\beta(t)} (s, \mathrm{d}_\cS, \oa) - F^0_{\beta_*} (s, \mathrm{d}_\cS, \oa) - \mu \sbr{F^0_{\beta(t)} (s', \mathrm{d}_\cS', \oa') - F^0_{\beta_*} (s', \mathrm{d}_\cS', \oa')} } \rbr{F^0_{\beta(t)} (s, \mathrm{d}_\cS, \oa) - F^0_{\beta_*}(s, \mathrm{d}_\cS, \oa)} \nom  \\
 \geq & (1- \mu) \EE_{\init, s, \mathrm{d}_\cS, a} \rbr{F^0_{\beta_*} (s, \mathrm{d}_\cS, \oa) - F^0_{\beta(t)}(s, \mathrm{d}_\cS, \oa)}^2, \label{ineq:sep_2}
\end{align}
where in the last line we used Cauchy-Schwartz inequality, together with the fact that $(s, \mathrm{d}_\cS, \oa)$ and $(s', \mathrm{d}_\cS', \oa')$ share the same distribution. That is 
\begin{align*}
& \EE_{\init, s, \mathrm{d}_\cS, \oa, s', \mathrm{d}_\cS',\oa'} \sbr{F^0_{\beta(t)} (s', \mathrm{d}_\cS', \oa') - F^0_{\beta_*} (s', \mathrm{d}_\cS', \oa')} \sbr{F^0_{\beta(t)} (s, \mathrm{d}_\cS, \oa) - F^0_{\beta_*}(s, \mathrm{d}_\cS, \oa)}\\
\leq & \rbr{\EE  \sbr{F^0_{\beta(t)} (s', \mathrm{d}_\cS', \oa') - F^0_{\beta_*} (s', \mathrm{d}_\cS', \oa')}^2 \EE\sbr{F^0_{\beta(t)} (s, \mathrm{d}_\cS, \oa) - F^0_{\beta_*}(s, \mathrm{d}_\cS, \oa)}^2}^{1/2}\\
= & \EE_{\init, s, \mathrm{d}_\cS, \oa}  \sbr{F^0_{\beta(t)} (s, \mathrm{d}_\cS, \oa) - F^0_{\beta_*} (s, \mathrm{d}_\cS, \oa)}^2.
\end{align*}

Combining \eqref{ineq:merge}, \eqref{ineq:sep_1}, and \eqref{ineq:sep_2}, the second term in \eqref{eq:convergence} can be bounded by
\begin{align}
& - \eta \EE \inner{\beta(t) - \beta_*}{ \hat{g}_{\beta(t)} - \overline{g}^0_{\beta_*}}  \nom\\
\leq & - \eta  \EE  \sbr{-R_\beta \norm{\overline{g}_{\beta(t)} - \overline{g}^0_{\beta(t)}}_2
 + (1- \mu) \EE \rbr{F^0_{\beta_*} (s, \mathrm{d}_\cS, \oa) - F^0_{\beta(t)}(s, \mathrm{d}_\cS, \oa)}^2 } \nom \\
  & + \eta R_\beta  \rbr{\EE \norm{\hat{g}_{\beta(t)} - g_{\beta(t)}}_2^2}^{1/2} \label{ineq:second}.
\end{align}

To bound the third term in \eqref{eq:convergence}, we have
\begin{align}\label{ineq:bound_reg}
\EE \norm{\hat{g}_{\beta(t)} - \overline{g}^0_{\beta_*}}_2^2 \leq 
3 \EE \rbr{   \norm{\hat{g}_{\beta(t)} - \overline{g}_{\beta(t)}}_2^2   +   \norm{\overline{g}_{\beta(t)} - \overline{g}^0_{\beta(t)}}_2^2 + \norm{\overline{g}^0_{\beta(t)} - \overline{g}^0_{\beta_*}}_2^2 }.
\end{align}
We can use  Lemma \ref{lemma:mse_ssg} and Lemma \ref{lemma:diff_pg} to control the first two terms in \eqref{ineq:bound_reg}.
We proceed to bound the third one, note that $\nabla F^0_{\beta}(s, \mathrm{d}_\cS, \oa) = \nabla F_{\beta(0)}(s, \mathrm{d}_\cS, \oa) $ for any $\beta$, and $\norm{\nabla F_{\beta(0)}(s, \mathrm{d}_\cS, \oa)}_2 \leq 1$.
\begin{align}
\norm{\overline{g}^0_{\beta(t)} - \overline{g}^0_{\beta_*}}_2^2  &\leq \sbr{\EE\rbr{ \delta^0_{\beta(t)} (s, \mathrm{d}_\cS, \oa, s', \mathrm{d}_\cS', \oa') -  \delta^0_{\beta_*} (s, \mathrm{d}_\cS, \oa, s', \mathrm{d}_\cS', \oa')) \norm{\nabla F^0_{\beta(0)}(s, \mathrm{d}_\cS, a} } }^2 \nom \\
& \leq \EE \sbr{ \rbr{\delta^0_{\beta(t)} (s, \mathrm{d}_\cS, \oa, s', \mathrm{d}_\cS', \oa') -  \delta^0_{\beta_*} (s, \mathrm{d}_\cS, \oa, s', \mathrm{d}_\cS', \oa')} \norm{\nabla F^0_{\beta(0)}(s, \mathrm{d}_\cS, a} } ^2 \nom \\
& \leq \EE \sbr{ \delta^0_{\beta(t)} (s, \mathrm{d}_\cS, \oa, s', \mathrm{d}_\cS', \oa') -  \delta^0_{\beta_*} (s, \mathrm{d}_\cS, \oa, s', \mathrm{d}_\cS', \oa'))  }^2  \nom \\
& \leq \EE \rbr{F^0_{\beta(t)} (s, \mathrm{d}_\cS, \oa) - F^0_{\beta_*} (s, \mathrm{d}_\cS, \oa) - \mu \sbr{F^0_{\beta(t)} ( s', \mathrm{d}_\cS', \oa') - F^0_{\beta_*} (s', \mathrm{d}_\cS', \oa') } }^2 \nom\\
& \leq 2(1+u^2) \EE \sbr{F^0_{\beta(t)} (s, \mathrm{d}_\cS, \oa) - F^0_{\beta_*} (s, \mathrm{d}_\cS, \oa) }^2 \label{ineq:sep_3_1},
\end{align}
where in the last inequality follows from $(s, \mathrm{d}_\cS, \oa)$ and $(s', \mathrm{d}_\cS' ,a')$ sharing the same distribution.

Combining \eqref{ineq:second}, \eqref{ineq:bound_reg} and \eqref{ineq:sep_3_1}, conditioned on $\beta(t)$, we have the following bound for \eqref{eq:convergence}:
\begin{align*}
& \EE \sbr{\norm{\beta(t+1) - \beta_*}_2^2 \vert \beta(t)} \\
   \leq & \norm{\beta(t) - \beta_*}_2^2 - \rbr{\eta (1-\mu) -6(1+\mu^2 \eta^2} \EE \sbr{F^0_{\beta_*} (s, \mathrm{d}_\cS, \oa) - F^0_{\beta(t)}(s, \mathrm{d}_\cS, \oa)}^2 \\
 + &  \eta R_\beta  \rbr{\EE \norm{\hat{g}_{\beta(t)} - g_{\beta(t)}}_2^2}^{1/2}
+  \eta R_\beta \norm{\overline{g}_{\beta(t)} - \overline{g}^0_{\beta(t)}}_2 \\
 + & 3\eta^2 \EE \rbr{   \norm{\hat{g}_{\beta(t)} - \overline{g}_{\beta(t)}}_2^2   +   \norm{\overline{g}_{\beta(t)} - \overline{g}^0_{\beta(t)}}_2^2}.
\end{align*}
From Lemma \ref{lemma:mse_ssg}, we have $\EE \norm{\hat{g}_{\beta(t)} - g_{\beta(t)}}_2^2 = \cO (R_\beta^2/N)$.
From Lemma \ref{lemma:diff_pg}, we have $\EE \norm{\overline{g}_{\beta(t)} - \overline{g}^0_{\beta(t)}}_2 = \cO \rbr{R_\beta^{3/2}/m^{1/4}}$.
From Lemma \ref{lemma:mse_ssg}, we have $ \EE \norm{\hat{g}_{\beta(t)} - \overline{g}_{\beta(t)}}_2^2 = \cO (R_\beta^2)$.
Hence
\begin{align*}
& \EE \norm{\beta(t+1) - \beta_*}_2^2 \\
  \leq & \norm{\beta(t) - \beta_*}_2^2 - \rbr{\eta (1-\mu) -6(1+\mu^2) \eta^2} \EE \sbr{F^0_{\beta_*} (s, \mathrm{d}_\cS, \oa) - F^0_{\beta(t)}(s, \mathrm{d}_\cS, \oa)}^2 \\
& +  \eta  \cO \rbr{ \frac{R_\beta^2}{N^{1/2}} + \frac{R_\beta^{5/2}}{m^{1/4}} }+  \eta^2 \cO(R_\beta^2).
\end{align*}

Re-arrange and telescope, we have
\begin{align}
&\sum_{t=0}^{T-1} \EE \sbr{F^0_{\beta_*} (s, \mathrm{d}_\cS, \oa) - F^0_{\beta(t)}(s, \mathrm{d}_\cS, \oa)}^2 \nom\\
 \leq &
\rbr{\eta (1-\mu) -6(1+\mu^2) \eta^2}^{-1} \EE \sbr{  \norm{\beta_{0} - \beta_*}_2^2 +  \eta  T  \cO \rbr{ \frac{R_\beta^2}{N^{1/2}} + \frac{R_\beta^{5/2}}{m^{1/4}} }+  T\eta^2 \cO(R_\beta^2) } \label{ineq:tele}.
\end{align}

Notice that  the left hand side of \eqref{ineq:tele} is convex with respect to $\beta(t)$, we can take average of $\{\beta(t)\}_{t=0}^{T-1}$.
Define $\overline{\beta}_T = \frac{1}{T} \sum_{t=0}^{T-1} \beta(t)$, 
with Jensen's inequality we have
\begin{align*} 
 & \EE \sbr{F^0_{\beta_*} (s, \mathrm{d}_\cS, \oa) - F^0_{\overline{\beta}_T}(s, \mathrm{d}_\cS, \oa)}^2\\
 \leq  &
 \rbr{\eta (1-\mu) -6(1+\mu^2) \eta^2}^{-1} \EE_\init \sbr{  \frac{\norm{\beta_{0} - \beta_*}_2^2}{T} +  \eta  \cO \rbr{ \frac{R_\beta^2}{N^{1/2}} + \frac{R_\beta^{5/2}}{m^{1/4}} }+  \eta^2 \cO(R_\beta^2) }.
\end{align*}
Now take $\eta = \cO (T^{-1/2})$, we have
\begin{align*}
 \EE \sbr{F^0_{\beta_*} (s, \mathrm{d}_\cS, \oa) - F^0_{\overline{\beta}_T}(s, \mathrm{d}_\cS, \oa)}^2 \leq  \cO \rbr{\frac{R_\beta^2}{T^{1/2}} + \frac{R_\beta^{5/2}}{m^{1/4}} + \frac{R_\beta^2}{N^{1/2}}   } .
\end{align*}

We can now apply Lemma \ref{lemma:func_diff} and conclude that:
\begin{align*}
 \EE \sbr{F^0_{\beta_*} (s, \mathrm{d}_\cS, \oa) - F_{\overline{\beta}_T}(s, \mathrm{d}_\cS, \oa)}^2 & \leq 2\EE \sbr{\rbr{F^0_{\beta_*} (s, \mathrm{d}_\cS, \oa) - F^0_{\overline{\beta}_T}(s, \mathrm{d}_\cS, \oa)}^2 + \rbr{F^0_{\beta_*} (s, \mathrm{d}_\cS, \oa) - F_{\overline{\beta}_T}(s, \mathrm{d}_\cS, \oa)}^2 }  \\
 & \leq \cO \rbr{\frac{R_\beta^2}{T^{1/2}} + \frac{R_\beta^{5/2}}{m^{1/4}} + \frac{R_\beta^2}{N^{1/2}}   + \frac{R_\beta^3}{m^{1/2}} }.
\end{align*}

\end{proof}

We can now specialize Theorem \ref{theorem:generic_opt} to policy optimization and policy evaluation.

\begin{proof}[Proof of Lemma  \ref{thrm:poly_eval}]
Note that $F_{\theta_*}$ satisfies $F^0_{\theta_*} = \Pi_{\cF^\cP_{R_\theta,m_Q}}(\zeta_{F^0_{\theta_* }})$. 
For policy evaluation, by definition \eqref{def:operator}, the operator $\zeta$ equals to the Bellman evaluation operator: $\zeta_{F_\theta}(s, \mathrm{d}_\cS, \oa) = \sbr{\cT^{\pi_k} F_\theta}(s, \mathrm{d}_\cS, \oa)$.
Hence we have $F^0_{\theta_*} = \Pi_{\cF^\cP_{R_\theta,m_Q}}(\sbr{\cT^{\pi_k} F^0_{\theta_*}})$.
Now by Assumption \ref{assump:richness}, we know that $\sbr{\cT^{\pi_k} F^0_{\theta_*}} \in \cF^\cP_{R_\theta,m_Q}$, together with the fact that $F^0_{\theta_*}  \in  \cF^\cP_{R_\theta,m_Q}$. 
By uniqueness of the projection, we have $F^0_{\theta_*} = \sbr{\cT^{\pi_k} F^0_{\theta_*}}$, which implies $F^0_{\theta_*} = Q^{\pi_k}$
The claim follows immediately by applying Theorem \ref{theorem:generic_opt}.
%
%
\end{proof}


\begin{proof}[Proof of Lemma \ref{thrm:poly_improv}]
Note that here  we have  $F^0_{\alpha_*} = \Pi_{F_{R_\alpha, m_A}}(\zeta_{F^0_{\alpha_* }})$. For policy optimization the operator $\zeta$ is defined by:
$\zeta_{F^0_{\alpha_* }}(s, \mathrm{d}_\cS, \oa) = \tau_{k+1}( \upsilon_k^{-1} F^{Q}_{\theta_k}(s, \mathrm{d}_\cS, \oa) + \tau_k^{-1} F^A_{\alpha_k}(s, \mathrm{d}_\cS, \oa))$.
Then we have
\begin{align*}
& \EE_{\tilde{\sigma}_k} \sbr{F^A_{ \overline{\alpha}(T)}(s, \mathrm{d}_\cS, \oa) - \tau_{k+1}( \upsilon_k^{-1} F^Q_{\theta_k}(s, \mathrm{d}_\cS, \oa) + \tau_k^{-1} F^A_{ \alpha_k}(s, \mathrm{d}_\cS, \oa)) }^2 \\
\leq & 2 \EE_{\tilde{\sigma}_k} \sbr{F^A_{ \overline{\alpha}(T)}(s, \mathrm{d}_\cS, \oa) -  F^0_{\alpha_*} (s, \mathrm{d}_\cS, \oa)}^2 +
 2 \EE_{\tilde{\sigma}_k} \sbr{  F^0_{\alpha_*} (s, \mathrm{d}_\cS, \oa) - \tau_{k+1}( \upsilon_k^{-1} F^Q_{\theta_k}(s, \mathrm{d}_\cS, \oa) + \tau_k^{-1} F^A_{ \alpha_k}(s, \mathrm{d}_\cS, \oa)) }^2  \\
 \leq & 
 2 \EE_{\tilde{\sigma}_k} \sbr{\tau_{k+1}( \upsilon_k^{-1} F^0_{\theta_k}(s, \mathrm{d}_\cS, \oa) + \tau_k^{-1} F^0_{ \alpha_k}(s, \mathrm{d}_\cS, \oa)) - \tau_{k+1}( \upsilon_k^{-1} F^Q_{\theta_k}(s, \mathrm{d}_\cS, \oa) + \tau_k^{-1} F^A_{ \alpha_k}(s, \mathrm{d}_\cS, \oa)) }^2\\
 & + 2 \EE_{\tilde{\sigma}_k} \sbr{F^A_{ \overline{\alpha}(T)}(s, \mathrm{d}_\cS, \oa) -  F^0_{\alpha_*} (s, \mathrm{d}_\cS, \oa)}^2 \\
 \leq &  \cO \rbr{\frac{R_\alpha^2}{T^{1/2}} + \frac{R_\alpha^{5/2}}{m_A^{1/4}} + \frac{R_\alpha^2}{N^{1/2}}   + \frac{R_\alpha^3}{m_A^{1/2}} },
\end{align*}
where the second inequality comes from the definition of projection and $F^0_{\alpha_*} = \Pi_{F_{R_\alpha, m_A}}(\zeta_{F^0_{\alpha_* }})$,
 and the last inequality comes from direct application of Theorem \ref{theorem:generic_opt} and Lemma \ref{lemma:func_diff}, together with the fact that $F^A_\alpha$ and $F^Q_\theta$ share the same initialization. 
\end{proof}

\begin{proof}[Proof of Theorem \ref{thrm:main}]
From Lemma \ref{lemma:err_propagation}, we know that 
\begin{align}
&  (1- \gamma) \rbr{\cL(\pi^*) - \cL(\pi_k)} \nonumber \\
  \leq &   \EE_{\nu_*} \big\{ \upsilon_k \sbr{  \mathrm{KL} \rbr{ \pi^*(\cdot|s, \mathrm{d}_\cS) \| \pi_{k}(\cdot|s, \mathrm{d}_\cS)}  - \mathrm{KL} \rbr{ \pi^*(\cdot|s, \mathrm{d}_\cS) \| \pi_{k+1}(\cdot|s, \mathrm{d}_\cS)} + \varepsilon_k + \varepsilon_k' }  + \upsilon_k^{-1} M \big\}. \label{ineq:telescope}
\end{align}
where $\epsilon_k$ and $\epsilon_{k+1}'$ are defined as in Lemma \ref{thrm:poly_eval} and Lemma \ref{thrm:poly_improv}, respectively. In addition, $\varepsilon_k = \tau_{k+1}^{-1} \epsilon_{k+1} \phi_{k+1}^* + \upsilon_k^{-1} \epsilon_k' \psi_k^*$,  $\varepsilon_k' = \abs{\cA} \tau_{k+1}^{-2} \epsilon_{k+1}^2$, and $M = \EE_{\nu^*} \sbr{ \max_{a \in \cA} (F^Q_{\theta_0}(s, \mathrm{d}_\cS, \oa))^2} + 2R_\alpha^2$.
 $\phi_k^*$ and $\psi_k^*$ are defined by:
\begin{align*}
\phi_k^* &= \EE_{\tilde{\sigma}_k} \sbr{\abs{\mathrm{d}\pi^* / \mathrm{d}\pi_0 -  \mathrm{d}\pi_{k} / \mathrm{d}\pi_0}^2 }^{1/2}, \\
\psi_k^* &= \EE_{\sigma_k} \sbr{\abs{\mathrm{d}\sigma^* / \mathrm{d}\sigma_k -  \mathrm{d}(\nu^*\times \pi_k) / \mathrm{d}\sigma_k}^2 }^{1/2}.
\end{align*}


Now sum up \eqref{ineq:telescope} from $k=0$ to $K-1$, with $\upsilon_k = \upsilon \sqrt{K}$, we have
\begin{align*}
& (1- \gamma) K \min_{0 \leq k \leq K-1} \rbr{\cL(\pi^*) - \cL(\pi_k)} \nonumber \\
  \leq &   \EE_{\nu_*} \big\{ \upsilon\sqrt{K} \sbr{  \mathrm{KL} \rbr{ \pi^*(\cdot|s, \mathrm{d}_\cS) \| \pi_{0}(\cdot|s, \mathrm{d}_\cS)}  - \mathrm{KL} \rbr{ \pi^*(\cdot|s, \mathrm{d}_\cS) \| \pi_{K}(\cdot|s, \mathrm{d}_\cS)} + \sum_{k=0}^{K-1}  (\varepsilon_k + \varepsilon_k') }   + \upsilon^{-1}\sqrt{K} M \big\}. 
\end{align*}
Since we initialize policy $\pi_0$ to be uniform policy, we have $  \mathrm{KL} \rbr{ \pi^*(\cdot|s, \mathrm{d}_\cS) \| \pi_{0}(\cdot|s, \mathrm{d}_\cS)} \leq \log \abs{\cA}$.
Rearrange, we obtain
\begin{align}
\min_{0 \leq k \leq K} \{ \cL(\pi^*) - \cL(\pi_k) \} \leq \frac{\upsilon^2\rbr{ \log|\cA| + \sum_{k=1}^{K-1} (\varepsilon_k + \varepsilon_k')} + M}{(1-\gamma)\upsilon \sqrt{K}} , \label{ineq:thrm_p1}
\end{align}

From Lemma  \ref{thrm:poly_eval} and Lemma  \ref{thrm:poly_improv}, we have
\begin{align*}
\epsilon_k  = \cO \rbr{\frac{R_\theta^2}{T^{1/2}} + \frac{R_\theta^{5/2}}{m_Q^{1/4}} + \frac{R_\theta^2}{N^{1/2}}   + \frac{R_\theta^3}{m_Q^{1/2}} },
\epsilon_{k+1}' 
=  \cO \rbr{\frac{R_\alpha^2}{T^{1/2}} + \frac{R_\alpha^{5/2}}{m_A^{1/4}} + \frac{R_\alpha^2}{N^{1/2}}   + \frac{R_\alpha^3}{m_A^{1/2}} }.
\end{align*}

To control $\varepsilon_k$ and $\varepsilon_k'$, we have
\begin{align*}
\tau_{k+1}^{-1} \epsilon_{k+1}'\phi_{k+1}^*& = 
\cO\rbr{ k K^{-1} \phi_k^* (R_\alpha^2 T^{-1/2} + R_\alpha^{5/2} m_A^{-1/4} + R_\alpha^2  N^{-1/2}} \\
\abs{\cA} \tau_{k+1}^{-2} (\epsilon_{k+1}')^2 &= \cO \rbr{k^2 K^{-1} \abs{\cA} (R_\alpha^2T^{-1/2} + + R_\alpha^{5/2} m_A^{-1/4} + R_\alpha^2 N^{-1/2})^2 } \\
\upsilon_k^{-1} \epsilon_k \psi_k^* & = \cO \rbr{K^{-1/2} \psi_k^*(R_\theta^2 T^{-1/2} + R_\theta^2 N^{-1/2} + R_\theta^{5/2}m^{-1/4}  },
\end{align*}
if $m_A = \Omega(R_\alpha^2)$ and $m_Q = \Omega(R_\theta^2)$. One can verify that with the choice of $N, T, m_A, m_Q$ in Theorem \ref{thrm:main}, we have $\varepsilon_k = \varepsilon_k' = \cO(K^{-1})$.
Plug into \eqref{ineq:thrm_p1}, we obtain the second part of Theorem \ref{thrm:main}.

\end{proof}

\end{document}